\PassOptionsToPackage{sort&compress}{natbib}
\PassOptionsToPackage{capitalise}{cleveref}
\documentclass[11pt]{style/prism-princeton}

\makeatletter
\newcommand{\longdash}[1][2em]{\makebox[#1]{$\m@th\smash-\mkern-7mu\cleaders\hbox{$\mkern-2mu\smash-\mkern-2mu$}\hfill\mkern-7mu\smash-$}}
\makeatother
\newcommand{\omitskip}{\kern-\arraycolsep}

\usepackage{algorithm}
\usepackage{needspace}
\graphicspath{{./}}
\makeatletter
\providecommand*{\theHALG@line}{}
\renewcommand*{\theHALG@line}{\thealgorithm.\arabic{ALG@line}}
\makeatother

\newcommand{\acknowledgments}[1]{\section*{Acknowledgments}#1}

\newcommand{\metric}[2]{\makebox[0pt][c]{{\fontsize{7pt}{7.4pt}\selectfont #1\(\pm\)#2}}}
\newcommand{\bestmetric}[2]{\makebox[0pt][c]{{\fontsize{7pt}{7.4pt}\selectfont\bfseries\boldmath #1\(\pm\)#2}}}
\newcommand{\arenascore}[2]{\makebox[0pt][c]{{\fontsize{7pt}{7.4pt}\selectfont #1\(\pm\)#2}}}
\newcommand{\bestarenascore}[2]{\makebox[0pt][c]{{\fontsize{7pt}{7.4pt}\selectfont\bfseries\boldmath #1\(\pm\)#2}}}

\author[1]{Samuel M. Bateman}
\author[1]{Tenny Yin}
\author[1]{Chongyi Zheng}
\author[1,2]{Lei Huang}
\author[1]{Brian Wang}
\author[1]{Benjamin Eysenbach}
\author[1]{Jaime Fernández Fisac}
\author[1]{Dhruv Shah}

\affiliation[1]{Princeton University}
\affiliation[2]{Columbia University}

\hypersetup{
  pdftitle={Training Controllable World Models via Action Influence Maximization},
  pdfauthor={Samuel M. Bateman, Tenny Yin, Chongyi Zheng, Lei Huang, Brian Wang, Benjamin Eysenbach, Jaime Fernández Fisac, Dhruv Shah},
  pdfkeywords={world models, reinforcement learning, empowerment, active learning}
}

\begin{document}
\title{Training Controllable World Models via Action Influence Maximization}
\abstract{Action-conditioned video generative models trained on successful, goal-directed demonstrations tend to ignore the commanded action, rendering a plausible expert-like future regardless of what is asked of them. We trace this failure to training data that lack sufficient \emph{action influence}. We analyze this metric and prove that it upper bounds the action influence that any model trained on that data can express. Expert demonstrations, whose narrow action distributions let next states be predicted without attending to the action, set a low ceiling that no amount of model capacity alone can raise. Raising it requires collecting different data. We introduce \emph{Action Influence Maximization} (AIM), an algorithm for actively collecting training data for improving dataset action influence. 
We compare the performance and data-efficiency of AIM against exploration heuristics and intrinsic-motivation methods on simulated navigation and manipulation tasks. In pixel-based manipulation, AIM is the only method that consistently improves controllability: 500 actively collected episodes added to 3k demonstrations match a world model trained on 300k expert episodes on end-effector controllability, an $85\times$ reduction in data.
}
\website{https://princeton-prism.github.io/action-influence-maximization/}{princeton-prism.github.io/action-influence-maximization}
\code{https://github.com/princeton-prism/action-influence-maximization}{github.com/princeton-prism/action-influence-maximization}
\maketitle

\section{Introduction}

\label{sec:intro}

Action-conditioned world models promise to predict what would happen if a robot took different actions at a given state, enabling policy improvement in novel problems. 
Yet, when queried with out-of-distribution actions, models trained on expert demonstrations can predict expert-like futures or hallucinate implausible alternatives.
This failure of controllability (the ability to faithfully predict how changing actions affect outcomes) limits their usefulness for improving robot policies.

Scaling the quantity and diversity of training data can mitigate this~\citep{yin2026playworldlearningrobotworld, guo2026vlawiterativecoimprovementvisionlanguageaction}, but is limited by the cost of collecting robotics data.
To make world models competitive with on-robot experience, the practical question is how to obtain the most useful world-model training data from each real robot interaction.
Despite this, optimal data collection for robotic world models is relatively underexplored compared to policy improving data~\citep{hejna2025robotdatacurationmutual, ross2011dagger, hu2025rac}.
This paper investigates two key hypotheses: Does the data collection distribution strongly influence the downstream world model controllability and fidelity? And, supposing it consequential, can we optimize the data collection distribution to directly improve world model controllability and fidelity?
This work suggests that the answer to both questions is yes. 

This has a plausible and intuitive explanation: expert demonstrations in high-dimensional action spaces (e.g. human teleoperation) achieve their goals through a narrow sequence of actions.
The more sensitive success is to action variability, the more deterministic expert demonstrations must be.
Successful demonstrations can thus provide insufficient variation in actions to reveal evidence of how alternative actions would change outcomes.
A world model fitting this data can rationally ignore action influence and usually predict expert-like outcomes regardless of action command~\citep{yin2026playworldlearningrobotworld}.
In contrast, a non-expert learner may visit different states and try actions underrepresented in expert demonstration.
Predicting consequences of varying actions in these states are precisely the situations where downstream policy improvement stands to benefit the most and where the world model has the least data support.
We see evidence for this (\cref{fig:scaling-policy-data,fig:fourrooms-pair-dozen-plots}): scaling passive data collection improves controllability for different eval distributions, but the most efficient collection strategy is eval-distribution dependent.

To address this, we propose Action Influence Maximization (AIM), an active data collection algorithm that uses a data-derived intrinsic reward.
Put simply, AIM estimates which variations in action consequences would be poorly explained by the current dataset using a theoretically-motivated, learned intrinsic reward, and then focuses data collection on those state-action pairs as a form of embodied active learning.
To do this, we connect conditional dataset mutual information (MI) to the controllability of any downstream world model, and derive a tractable learned pointwise reward estimator using data density estimation to optimize this property directly.
AIM achieves up to an $85\times$ improvement in data efficiency in a high-dimensional, robot manipulation task: adding just 500 actively collected episodes to 3k expert demonstrations matches the end-effector controllability of a world model trained on 300k expert episodes.

\begin{figure}[tb]
    \centering
    \includegraphics[trim=0bp 0bp 0bp 160bp,clip,width=0.9\textwidth]{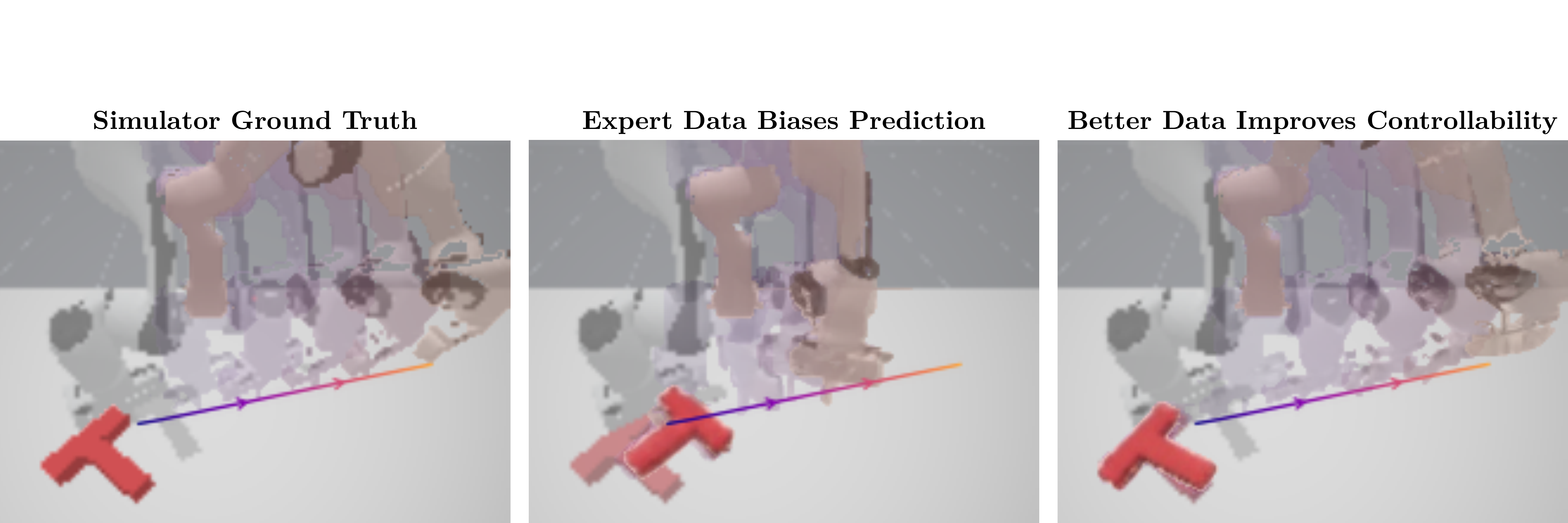}
    \caption{Action-Conditioned Video Generative Models trained on near-expert data suffer from expert success bias and low controllability. The center panel shows a world model trained on expert data ignoring input actions and hallucinating a state closer to success. The right panel shows a world model trained with AIM data which remains faithful to the commanded action.}
    \label{fig:generation-variation-placeholder}
\end{figure}

\begin{figure}[tb]
    \centering
    \includegraphics[trim=0bp 0bp 0bp 85bp,clip,width=0.9\textwidth]{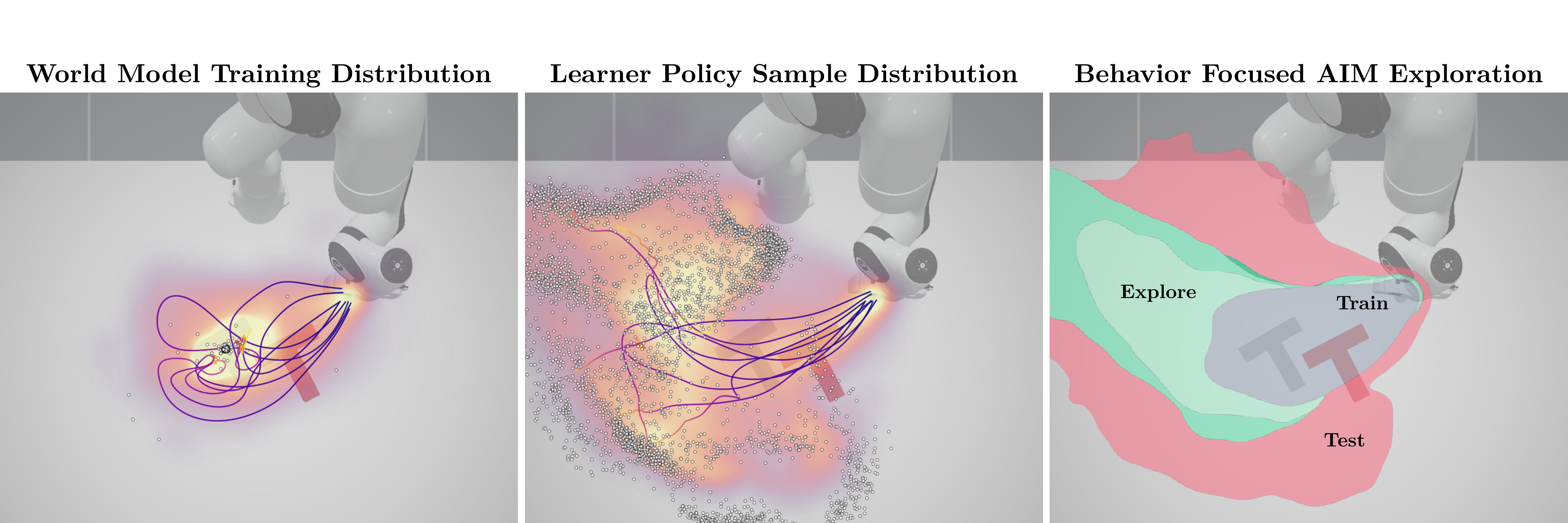}
    \caption{AIM is an algorithm that uses an information maximization objective to collect high-utility experimental transitions to improve controllability of learner policies similar to the prior demonstrations.}
    \vspace{-5pt}
    \label{fig:teaser}
\end{figure}

\section{Related Work}

\label{sec:related-work}

Learning a model to predict the future consequences of actions is a well studied problem in decision-making~\citep{sutton1991dyna, dayan1993successorrepresentation, watter2015embed2ctrl, barreto2017successorfeatures, hafner2019planet, Schrittwieser2019MasteringAG, Hafner2020Dream, Hansen2022tdmpc, wu2022daydreamer}. 
Recent advances in generative modeling~\citep{Ho2020DDPM, song2021DDIM, lipman2023flow, chen2025diffusionforcing, ai2026jointlikelihoodflowmaps, rombach2022stablediffusion, blattmann2023stablevideodiffusion} and self-supervised representation learning~\citep{balestriero2025lejepaprovablescalableselfsupervised, maes2026leworldmodel, wang20261000, zheng2024stabilizing} have significantly improved the scalability of learned predictive models, enabling a new family of action-conditioned generative world models~\citep{hafner2025dreamerv4, guo2026ctrlworld, yang2024unisim, wang2026interactiveworldsimulatorrobot, nvidia2026worldsimulationvideofoundation, gao2026dreamdojogeneralistrobotworld}. 
This family of methods absorbs heterogeneous internet-scale video and robotics datasets, and aims to maximize distributional similarity between generated videos and the dataset~\citep{heusel2017fid,unterthiner2019fvd,sutherland2017generative,ge2024fvdbias}.
In robotics, this goal is misaligned: these models' main utility is improving learner policies~\citep{wu2022daydreamer, guo2026vlawiterativecoimprovementvisionlanguageaction, guo2026ctrlworld, quevedo2026worldgym, wang2026interactiveworldsimulatorrobot, Hansen2022tdmpc, maes2026leworldmodel, zhu2025uwm, ye2026dreamzeroworldactionmodelszeroshot}, where a rollout is only useful if it faithfully reproduces action-induced changes in task-relevant features that are generally unknown at training time.
Our studied world model builds on DreamerV4~\citep{hafner2025dreamerv4} and Diffusion Forcing~\citep{chen2025diffusionforcing}, with architectural and training details described in \Cref{apx:wm-arch}.

A surrogate for task-relevant fidelity is a world model's \emph{controllability}: faithful, action-responsive simulation over long horizons~\citep{guo2026vlawiterativecoimprovementvisionlanguageaction, yin2026playworldlearningrobotworld}. Prior work has improved this fidelity by improving temporal consistency~\citep{chen2025diffusionforcing, song2025historyguided}, incorporating data from autonomous play~\citep{yin2026playworldlearningrobotworld} or failure~\citep{guo2026vlawiterativecoimprovementvisionlanguageaction}, enforcing policy conservatism to avoid problematic extrapolation~\cite{hafner2025dreamerv4}, performing self-consistency based post-training \cite{ye2025reinforcementlearninginverserewards, wang2026evaaligningvideoworld}, or using latent representations that are easier to predict
\citep{maes2026leworldmodel, assran2025vjepa2selfsupervisedvideo, nilaksh2026reconstructionsemanticsmakeslatent}. 
Still, techniques from other domains have seen limited application to modern world models, such as active data collection~\citep{Sharma2020dads, sekar2020plan2explore, hejna2025robotdatacurationmutual, Huang-2025-fail2progress, Hansen2026Hallucination} or sim-and-real co-training~\citep{lei2026mechanistic}.

Exploration has a long history in reinforcement learning (RL)~\citep{bellemare2016unifying, pathak2017curiosity, burda2018exploration, ecoffet2019go, pathak2019self} and MI skill learning~\citep{eysenbach2018diyan, Sharma2020dads, park2024metra, zheng2025can}.
In model-based RL, prediction errors~\citep{stadie2015incentivizing, pathak2017curiosity, burda2018exploration} or ensemble disagreement~\citep{pathak2019self, sekar2020plan2explore} provide an intrinsic reward for collecting data that a particular model struggles to predict.
In contrast, in active learning~\citep{Settles2009activelearninglitreview}, it is sometimes assumed that a small \textit{core set} of data can efficiently summarize a data distribution~\citep{sener2018coreset}.
\citet{toneva2018exampleforgetting}'s ``unforgettable examples'' motivate this perspective: sample learning difficulty and utility is largely consistent across models.
This suggests that data utility is intrinsic to the data, independent of the derived model, and can guide collection~\citep{kendall2017uncertainty, yoo2019lossprediction}.
This is similar to information-theoretic experimental design~\citep{lindley1956bayesianexperimentaldesign, chaloner1995bayesianexperimentaldesignareview} and persistent excitation in adaptive control~\citep{boyd1986necessary, demetriou1993persistence}.
Our work is most closely related to PlayWorld~\citep{yin2026playworldlearningrobotworld}, a method for refining generative world models via on-policy goal-directed ``play data" rollouts, and contemporaneous work by \citet{Hansen2026Hallucination} which uses heuristic measures of prediction instability as a curiosity signal to guide active collection.
In contrast, we propose that the model and its data are largely separable, and optimize the data itself by maximizing intrinsic dataset action influence.

\section{Data Collection with Action Influence Maximization}
\label{sec:aim}
To turn this intuition into a principled collection algorithm, we first quantify world model controllability and connect it to dataset action information. Then, we derive an intrinsic reward to maximize this quantity, and provide a practical batch online RL implementation of the theoretical result.

\subsection{How does data collection shape world model controllability?}
\label{sec:aim-motivation}
\label{sec:story-acquisition-hypothesis}

A general proxy for task-specific simulation fidelity is to maximize the controllability of
the world model: if the world model accurately represents how different
actions influence the next state, and downstream rewards are
primarily evaluated on features which are at least somewhat controllable, 
then maximizing the action controllability of the learned state 
transition should act as a valid surrogate for any task-specific fidelity we choose.
Formally, we define a measure of world model action controllability as the
difference in prediction error when the chosen action is known or not,
dubbed the \emph{explained action information score}:
{
\begin{equation}
\label{eq:story-action-controllability}
    \hat I(\rho_{\mathrm{eval}},\pi_{\mathrm{eval}})
    \mathrel{:=}
    \mathbb{E}_{d_{\mathrm{eval}}(s,a)}
    \left[
        \underbrace{
            D_{\mathrm{KL}}\!\left(
                p(\cdot\mid s,a)\,\middle\|\,
                \hat p_{\theta,\pi_{\mathrm{eval}}}(\cdot\mid s)
            \right)
        }_{\text{prediction error without the action}}
        -\underbrace{
            D_{\mathrm{KL}}\!\left(
                p(\cdot\mid s,a)\,\middle\|\,
                \hat p_\theta(\cdot\mid s,a)
            \right)
        }_{\text{prediction error with the action}}
    \right].
\end{equation}
\par
}

\begin{wrapfigure}{R}{0.28\textwidth}
    \centering
    \captionsetup{font=footnotesize}
    \includegraphics[width=\linewidth]{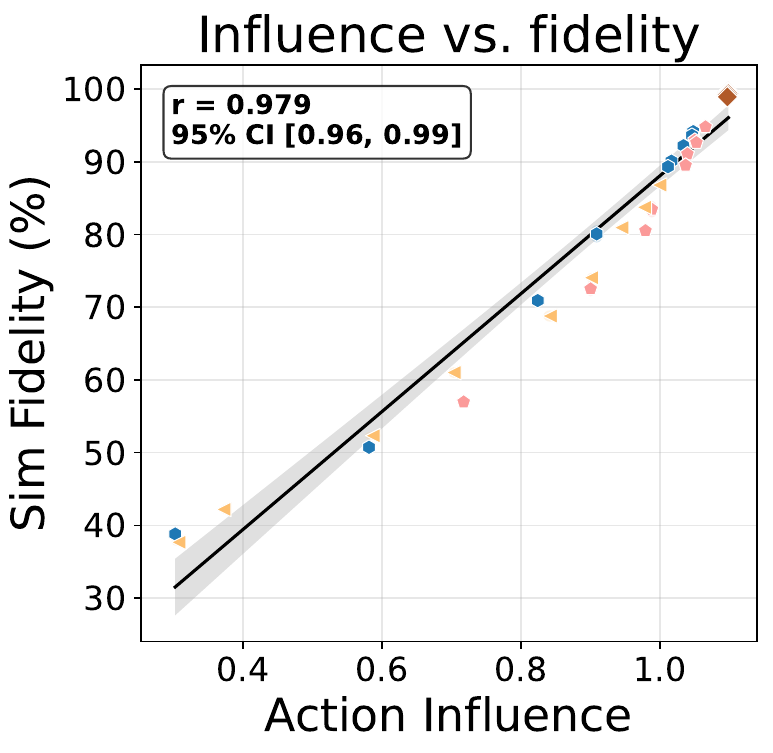}
    \caption{\textbf{Action information and fidelity.} Explained action information correlates with state-extracted simulation fidelity in Four Rooms, across dataset scales (\(r=0.979\)).}
    \label{fig:fourrooms-action-information-fidelity}
\end{wrapfigure}

Without knowledge of the action,
the model averages its predictions over the evaluation policy:
\(\hat{p}_{\theta,\pi_{\mathrm{eval}}}(s'\mid s)
= \int_{\mathcal{A}}
\pi_{\mathrm{eval}}(a\mid s)\,\hat{p}_\theta(s'\mid s,a)\,da\),
where \(s,s'\) are MDP states, \(a\in\mathcal A\) is an action,
\(p\) is the true transition kernel, and \(\hat p_\theta\) its learned
approximation. The evaluation occupancy is
\(d_{\mathrm{eval}}(s,a)=\rho_{\mathrm{eval}}(s)\pi_{\mathrm{eval}}(a\mid s)\),
where \(\pi_{\mathrm{eval}}\) represents the mixture of learner policies
encountered during policy improvement and \(\rho_{\mathrm{eval}}\) their
aggregate state occupancy.
A larger score implies a world model more accurately predicts transitions
when it sees the action than if it was held out, plausibly improving action
responsive task-relevant transition fidelity indirectly.
The observed correlation between explained action information and simulation fidelity supports this interpretation (Figure~\ref{fig:fourrooms-action-information-fidelity}).

Our data scaling experiments in \cref{fig:scaling-policy-data}
indicate that most data collection strategies will saturate in the data limit
long before saturating the simulation fidelity of the downstream task evaluation.
Near-deterministic expert policy data provides few examples of alternate actions
in a state (Figure~\ref{fig:pusht-policy-mode-3x2}), making
\(\hat{p}_{\theta,\pi_{\mathrm{eval}}}(s'\mid s)
\approx\hat{p}_\theta(s'\mid s,a)\)
an equally plausible but degenerate explanation of the data.

If world model fidelity is mostly about how well its training distribution covers
\(d_{\mathrm{eval}}(s,a)\), then expert-oriented data on one
task might be maximally bad for improving a weak learner.
A weak learner definitionally concentrates very little of its state-action occupancy
on highly rewarding states, but a world model trained on expert-like data
will primarily predict plausible expert-like futures with action
insensitivity or blindly extrapolate.
Figure~\ref{fig:fourrooms-pair-dozen-plots} illustrates this action
insensitivity at low data and its improvement with more expert data.
Adding failure data is likely an insufficient remedy: the world model must
accurately model every learner evaluation distribution, from the weakest
policy to the final strongest. Failure data adds a distribution in between,
which may or may not bridge the gap between the learner's distribution and the expert distribution.

\begin{figure}[tb]
    \centering
    \includegraphics[width=0.95\textwidth]{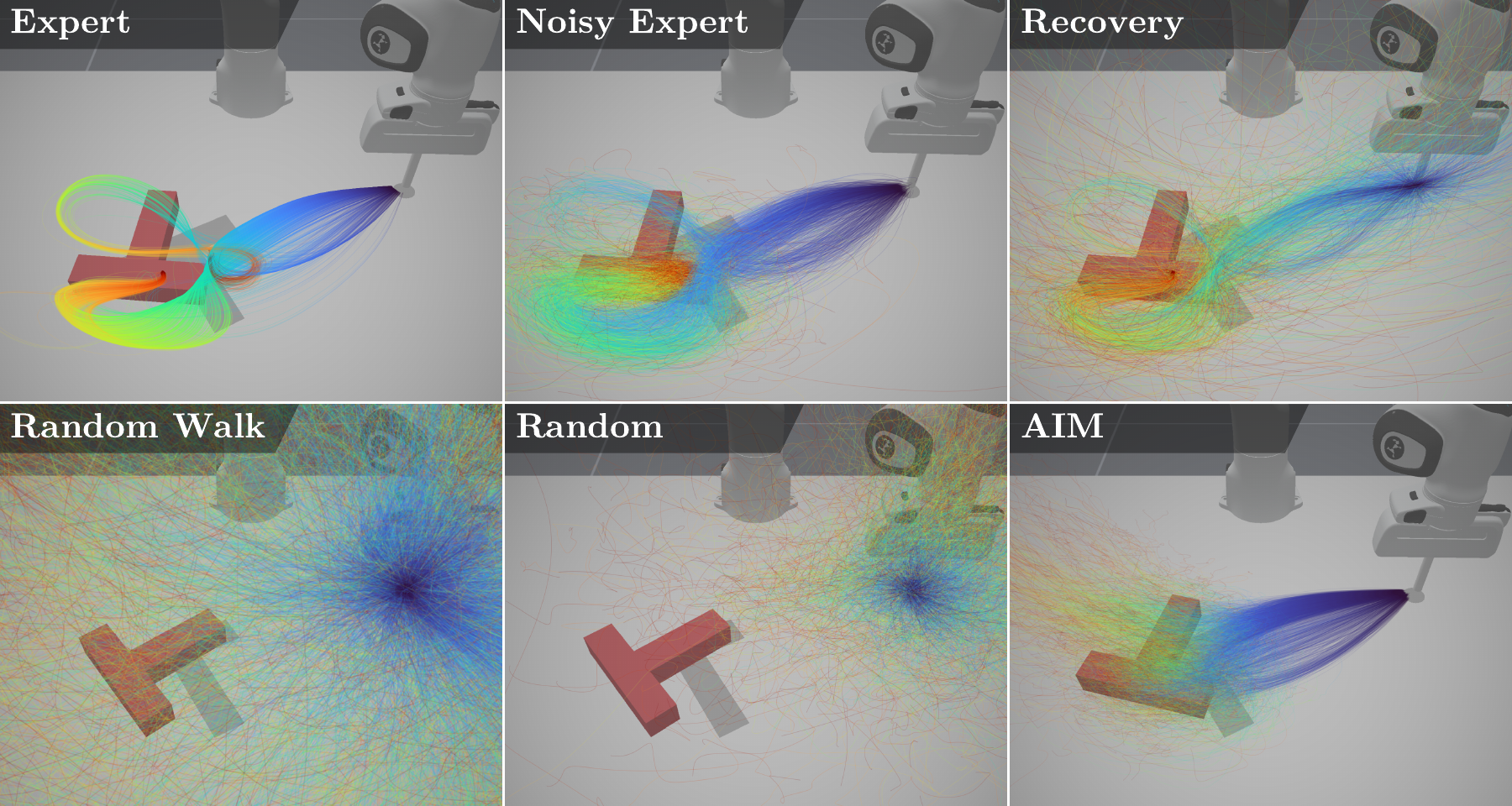}
    \caption{
    \textbf{State-Action Coverage.} Near-expert policies concentrate coverage around successful task execution, while random policies provide broader but less structured coverage. AIM targets action-informative regions useful for improving world model controllability.
    }
    \vspace{-15pt}
    \label{fig:pusht-policy-mode-3x2}
\end{figure}

This results in a sharp, testable hypothesis:
passive data scaling is perhaps only as effective as the data's natural
coverage of evaluation states and the action consequences revealed at those states.
Such incidental coverage can have very poor sample complexity in
high-dimensional search spaces.
In contrast, data-driven active data collection can focus scarce data sampling on
high-utility transitions to maximize coverage of diverse
state-action pairs, avoid duplication of data, and direct exploration
toward ``interesting'' data.
A reward derived from our controllability metric could guide this
exploration in a principled manner.

\subsection{An Objective for Collecting Informative Data}
\label{sec:aim-contextual-objective}

Data plays such a large role in world model capabilities (or lack
thereof) that it seems possible that
an upper bound on a world model's capabilities could be
quantified from the data alone. A key observation
using Bayes' rule and KL nonnegativity allows us to quantify
the intrinsic dataset controllability as conditional mutual
information, which upper-bounds the explained action information
in Equation~\eqref{eq:story-action-controllability}:

\begingroup
\begin{theorem}[Action information bounds model-explained controllability]
\label{thm:story-data-bound}
For a state distribution \(\rho\) and policy \(\pi\), the explained action
information in Equation~\eqref{eq:story-action-controllability} is bounded
by dataset conditional mutual information:
\begin{equation}
\label{eq:story-data-bound}
    \hat I(\rho,\pi)
    \leq I(\rho,\pi)
    \mathrel{:=}
    \mathbb E_{\rho(s)\pi(a\mid s)}
    \left[
        D_{\mathrm{KL}}\!\left(
            p(\cdot\mid s,a)\,\middle\|\,p_\pi(\cdot\mid s)
        \right)
    \right] 
    = I_{\rho\pi p}(A;S'\mid S),
\end{equation}
where \(p_\pi(s'\mid s)=\int_{\mathcal A}\pi(a\mid s)p(s'\mid s,a)\,da\)
is the true transition distribution averaged over actions.
\end{theorem}
Proof and regularity conditions: Appendix~\ref{app:story-data-bound-proof}.
\par\endgroup

This yields two key observations: world model capability
can be analyzed from properties of the dataset before training a world
model, and, when the bound is tight, world model
controllability can be directly optimized by improving the
return on the corresponding conditional mutual information maximizing
contextual bandits problem
\begingroup
\begin{equation}
\label{eq:story-contextual-optimization}
\begin{aligned}
    \underset{\pi}{\operatorname{maximize}}\quad
    &I(\rho_{\mathrm{eval}},\pi)
    &= \mathbb E_{s\sim\rho_{\mathrm{eval}}}
    \left[
        \mathbb E_{a\sim\pi(\cdot\mid s)}
        \left[
            D_{\mathrm{KL}}\!\left(
                p(\cdot\mid s,a)\,\middle\|\,p_\pi(\cdot\mid s)
            \right)
        \right]
    \right],
\end{aligned}
\end{equation}
\par\endgroup
selecting \(\pi\) to maximize \(I(\rho_{\mathrm{eval}},\pi)\)
on the evaluation distribution, with $\rho_{\mathrm{eval}}$ held fixed.

Since, in general, $\pi_{\mathrm{eval}}$ is not known at training time in the open-ended task policy improvement setting, this optimization problem can be seen as maximizing the diversity of distinguishable action consequences we can achieve from the evaluation state distribution $\rho_{\mathrm{eval}}$.
This is equivalent to maximizing a single-step conditional mutual information between actions and next states given the current state, which is identical to the \textit{one-step empowerment} objective~\citep{klyubin2005empowerment, salge2013empowermentintroduction}.
This provides an interesting insight: without additional inductive biases to enable extrapolation, reliable evaluation of learners requires the dataset to reveal the action consequences encountered under the learner's evaluation distribution.
Scaling passive data collection (e.g. passive video) is unlikely to recover missing action consequences without additional interactive data collection to quantify the effects of an agent's actions on those outcomes.
Equation~\eqref{eq:story-contextual-optimization} therefore provides a framing for improving world model data across downstream tasks.

\subsection{The AIM Reward}
\label{sec:aim-reward}

Equation~\eqref{eq:story-contextual-optimization} motivates an intervention
to test our hypothesis that active data collection can address the problem
of passive incidental coverage by reframing the optimization problem of
Equation~\eqref{eq:story-contextual-optimization} into a tractable active
collection surrogate. Since we are trying to perform a targeted search
for the data we are missing, we need to rely on data we have to
identify those absences.
This data generally has the form of a prior data distribution of observed
transitions, often called a replay buffer,
\(d_k(s,a,s')=d_k(s)\beta_k(a\mid s)p(s'\mid s,a)\), which contains the
previously visited data. Using this, we propose the following data
density-based reward for AIM, combining a
reference state coverage and mutual information term:

\begingroup
\begin{equation}
\label{eq:story-aim-reward}
    r_{\mathrm{AIM}}^k(s,a,s')
    :=\underbrace{\log q_k(a\mid s,s')-\log\beta_k(a\mid s)}_
        {{\text{$r_{\mathrm{MI}}$ Action-Informative Transitions}}}
      +\underbrace{\log\rho_{\mathrm{eval}}(s)-\log d_k(s)}_
        {{\text{$r_{\mathrm{cov}}^k$ Reference State Coverage}}}.
\end{equation}
\par\endgroup
{Here \(q_k(a\mid s,s')\) is the inverse dynamics under
the current replay policies \(\beta_k\).} 
Its utility is intuitive: the reference state coverage term will push the collected
data to be closer to the eval state distribution and away from previously collected data, and
the pointwise mutual information term rewards transitions
whose outcomes make the chosen action more identifiable than under the
existing policy. But it has a deeper connection to the mutual
information contextual bandits problem itself:

\begin{theorem}[AIM improves dataset mutual information]
\label{thm:story-replay-optimizer}
Under suitable conditions outlined in Appendix~\ref{app:story-replay-proof},
the Barber--Agakov lower bound~\citep{barberagakovlowerbound2003} on \(I(\rho,\pi)\) shares the same
optimizer as the replay buffer contextual bandits problem,
given by
\begingroup
\begin{equation}
\label{eq:story-boltzmann-optimizer}
    \pi_k^\star(a\mid s)
    \propto\beta_k(a\mid s)\exp (
        \mathbb E_{s'\sim p(\cdot\mid s,a)}
        [r_{\mathrm{AIM}}^k(s,a,s')]
    ),
\end{equation}
\endgroup
with normalization over actions at each state.
If \(\beta_k\ne\pi_k^\star\) in states with non-zero evaluation probability:
\begingroup
\begin{equation}
\label{eq:story-replay-improvement}
    I(\rho,\pi_k^\star)>I(\rho,\beta_k).
\end{equation}
\endgroup
\end{theorem}
Proof: Appendix~\ref{app:story-replay-proof}.

Thus optimizing the AIM reward in the contextual bandits setting
iteratively increases dataset mutual information, and thus an upper bound
on world model controllability.

Because the update reweights the prior replay data, expert demonstrations
bias initial collection toward informative variations
around expert behaviors in regions of interest. Optimizing this iteratively
can be seen as a guided search for action consequences:
it first explores variations around what humans have
revealed is interesting, then spreads outward as those
consequences become represented in the replay buffer and the re-estimated
reward favors alternatives that remain underrepresented. This provides a
way to progressively bridge the expert distribution to varying degrees
of degraded policies, the exact qualities of data we desired for accurate policy
evaluation.

\subsection{Estimating and Optimizing the AIM Reward in Practice}
\label{sec:implementation-details}

Directly optimizing the AIM contextual bandits problem of Theorem~\ref{thm:story-replay-optimizer} is not practical: it requires the ability to sample from the evaluation distribution \(\rho_{\mathrm{eval}}(s)\) arbitrarily and sample a transition from the true dynamics \(p(s'\mid s,a)\) for each state-action pair. 
Given the structure of equation~\eqref{eq:story-boltzmann-optimizer}, we can instead estimate the AIM reward function from the replay buffer and optimize a conservative sequential RL problem as a surrogate for the contextual bandits problem.
Estimating the AIM reward is challenging for three reasons: it's non-stationary as further data is collected, requires sufficiently accurate likelihood estimation of $q_{\beta_k}(a|s,s')$, $\beta(a|s)$, and $d_{k}(s)$, and the evaluation distribution $\rho_{\mathrm{eval}}(s)$ is not generally known at training.

To address the non-stationarity, we estimate the learned AIM reward function in a batch online RL loop which collects a batch of experimental data, updates learned density reward models to account for the new trajectories, and performs conservative RL to extract an AIM reward maximizing policy solution similar to equation~\eqref{eq:story-boltzmann-optimizer}.
This is described in detail in Algorithm~\ref{alg:aim-collection-implementation}.
This approach approximates the learned densities as piecewise constant in collection time to allow for amortizing the training of larger learned flow matching reward models~\citep{lipman2023flow,Chi2023DiffusionPolicy}.
This can be thought of abstractly as linearizing around the current dataset (refitting the density reward models), taking a conservative step along an improvement direction under this linearization (data collection with the reward maximizing policy), and then relinearizing at the new replay buffer dataset (refitting again).
Periodic re-estimation ensures the learned reward remains aligned with dataset utility as the data distribution evolves during collection.
We approximate the policy improvement in \Cref{eq:story-boltzmann-optimizer} by performing per-iteration value estimation via offline RL through an Implicit Q-Learning Critic \cite{kostrikov2022iql} and use AWR \cite{peng2019awr} for policy extraction into a diffusion policy actor~\citep{Chi2023DiffusionPolicy}.

For the reward estimation itself, this type of variational mutual information maximization has deep roots in pre-existing theoretical literature \cite{barberagakovlowerbound2003}, but practical estimation of density likelihoods over high dimensional data has been historically challenging.
We make a number of principled decisions to inherit the strengths of modern generative modeling to our task.
We estimate all likelihoods with a flow matching model \cite{lipman2023flow} with the Continuous Normalizing Flows FFJORD-Hutchinson estimator \cite{grathwohl2018ffjord, Hutchinson1989estimator} and amortize the divergence estimation cost through training time self-distillation \cite{ai2026jointlikelihoodflowmaps}.
This has a nice property of averaging single point Hutchinson sample estimates in the weights throughout training, making the likelihood estimation more reliable.
While this allows for exact likelihood and thus reward estimation, the practical differences in dimensionality between the state and action likelihoods creates numerical problems with balancing and training an RL critic on the AIM reward function. 
We address this by statistical standardization of the mutual information and reference state coverage terms, and clip AWR weights.
This reduces the sharpness of rare tail support events and improves effective sample size in AWR, while still giving utility-driven rewards.
Theoretically this can be interpreted as adjusting the sharpness of the weighting of the reference state coverage term, and as a trust region on the mutual information reward.

Finally, in absence of a specific desired evaluation distribution, we make $\rho_{\mathrm{eval}}(s)$ a uniform distribution over the target state space.
We find this simplification works reasonably well, but future work could explore more explicitly designing this distribution to the target learner distribution to further improve sample efficiency.

\section{Evaluation}
\label{sec:experiments-and-results}
\label{sec:story-scaling-experiment}

Section~\ref{sec:aim} proposes a compelling theoretical framework for analyzing world model controllability and its relation to data: the dataset mutual information (or lack of it) determines the environment faithful controllability exposed to a world model and largely determines its useful fidelity for policy learners.
To analyze this framework empirically, we study how the collection distribution affects world model fidelity, whether AIM improves data efficiency, and which parts of AIM contribute to that improvement.
We hold architecture and training details fixed within each of the following environments:

\textbf{Four Rooms:}
Four Rooms~\citep{nikulin2024xlandminigrid, MinigridMiniworld23} provides a fully observable, deterministic environment where we can exhaustively evaluate every state-action transition.
We train identical flow-matching world models~\citep{lipman2023flow, peebles2023dit} from scratch on several expert policy datasets: uniform random actions, sim-privileged Oracle PPO~\cite{schulman2017proximalpolicyoptimizationalgorithms}, partially observable PPO, and Early Stopped PPO.
We evaluate simulator fidelity by querying the learned model at all state-action pairs for an environment and matching predicted observations to all sim states via maximum likelihood.
This setting isolates the role of the transition distribution from partial observability, data contamination from large scale training, and bias towards any specific evaluation policy distribution.

\textbf{Push-T:}
Push-T tests whether the same data-collection principles apply in high dimensional, pixel-based robot manipulation.
We train a DreamerV4-Like latent Diffusion Forcing model~\citep{hafner2025dreamerv4,chen2025diffusionforcing,rombach2022stablediffusion,lipman2023flow,Kingma2013AutoEncodingVB} from scratch and compare datasets collected by expert, random, random-walk, synthetic RaC-like recovery~\citep{hu2025rac}, intrinsic motivation~\citep{eysenbach2018diyan, Sharma2020dads, park2024metra, pathak2017curiosity}, and AIM policies.
Exhaustive transition evaluation is infeasible in high-dimensional manipulation; we evaluate evenly sampled autoregressive horizons (H=10) on two ``learner'' target policy distributions: rollouts from a weak on-policy Behavior-Cloned (BC) Diffusion policy~\citep{Chi2023DiffusionPolicy} and an ``Anti-Push-T'' Push-Away policy which is trained to disperse the Tee to unique positions away from the normal goal.
Our controllability metric is end-effector masked pixel recall to isolate action-dependent controllability. Additional metrics, details on the data generating policies, and world model architectures are reported in the Appendix.

\subsection{Mutual Information and Controllability}
\label{sec:scaling-controllability-results}

First we briefly analyze the presumption that the dataset mutual information limits the model controllability.
We train 68 world models, 8 data scales in Push-T and 9 in Four Rooms, with 4 passive data collection strategies for each.
\Cref{fig:fourrooms-action-information-fidelity} summarizes the results from Four Rooms: across data scaling, the primary predictor of simulator fidelity is the world models explained action influence ($r=0.979$).
Push-T shows similar scaling behavior (see \Cref{fig:scaling-policy-data}), though in the data limit, different data collection policies converge to different levels of controllability.
This is compatible with our model: different fixed collection distributions will leave different portions of the state space systematically under sampled regardless of the data's utility.
Even random behavior becomes less effective as the dimensionality of the state space grows and the initial state distribution becomes less uniform (see \Cref{app:fourrooms-fixed-start-coverage}).
This motivates active data collection both empirically and theoretically.

\subsection{Active Data Collection with AIM}
\label{sec:results}

Now that we have established that passive data collection has diminishing returns beyond a certain point in modeling high dimensional control, we can evaluate another prediction of our theory: maximizing action information through active collection should provide a substantial boost in data efficiency for world model improvement.
We implement the approximate AIM sequential collection algorithm from Section~\ref{sec:implementation-details}, and collect 100-episode batches over four rounds in Four Rooms and five in Push-T to augment a base expert collection distribution of \(\sim 3\mathrm{K}\) expert demonstrations.
We compare this at equal rollout budget with both the active and passive data collection policies described in Section~\ref{sec:story-scaling-experiment}, training a total of 73 world models, and report our results in \Cref{fig:main-aim-results}.

\begin{figure*}[tb]
    \centering
\includegraphics[width=0.95\textwidth]{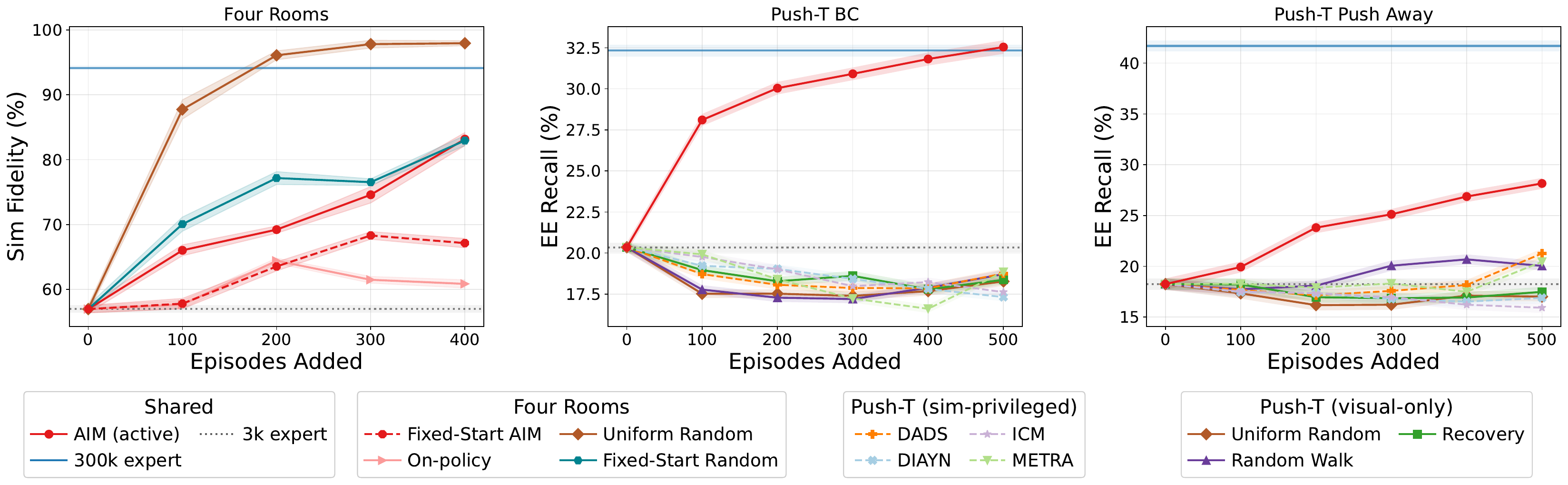}

    \caption{
    \textbf{Improving Controllability via Active Data Collection.} Across discrete-action navigation (FourRooms: left) and continuous-action manipulation (Push-T: center, right), AIM improves controllability with just a few hundred episodes. While \texttt{Random} is ideal when uniform state sampling is possible (FourRooms), AIM approaches this performance with more data; in more realistic scenarios (Push-T), AIM is the only method that improves controllability. Most notably, AIM with just 500 new episodes matches the performance of a world model trained with 300K expert episodes.
    }
    \vspace{-15pt}
    \label{fig:main-aim-results}
\end{figure*}

In Four Rooms, uniform random is predictably strongest, consistent with \Cref{fig:scaling-policy-data} and our prior analysis.
However, fixed-start Random performs much worse, showing that random's advantage comes largely from its initial state coverage rather than action entropy alone.
In Push-T, random, recovery, and intrinsic motivation baselines plateau or degrade on end-effector controllability, while AIM improves consistently between collections.
With 3k expert trajectories plus 500 AIM collections, the model matches the 300k-expert baseline on BC rollouts controllability, and approaches the 19k-expert baseline on Anti-Push-T rollouts.
The flat baselines do not indicate failed exploration.
These policies collect meaningful data and perform competitively on other fidelity metrics (see \Cref{fig:aim-bc-pusht-extended-results,fig:aim-anti-pusht-extended-results}) and at larger data scales, but their additional rollouts appear to transfer poorly to the target controllability metrics in the low-data regime.

Our theory suggests why broader data coverage need not improve downstream world-model utility, especially in the low data regime: alternative collection policies may add trajectories far from the expert distribution, creating ``islands'' of action consequences that the model struggles to connect to existing data, leaving intermediate evaluation behaviors poorly supported.
AIM can instead be thought of as ``growing'' a successively broader region of action explainability around existing data.
For policies which vary over learning, or alternate tasks with behavior similar to the prior data, this data ``bridge'' may provide better support than broader but disconnected exploration.
Two questions follow naturally: which parts of AIM contribute to this improvement, and how well does it generalize across conditions?

\subsection{Ablations}
\label{sec:aim-ablations}

\begin{table*}[!b]
  \centering
  \resizebox{\linewidth}{!}{\begin{minipage}{43pc}
  \begingroup
  \fontencoding{OT1}\fontfamily{ptm}\selectfont
  \renewcommand{\footnotesize}{\fontsize{8pt}{9pt}\selectfont}
  \renewcommand{\scriptsize}{\fontsize{7pt}{8pt}\selectfont}
  \setlength{\parskip}{0pt}
  \setlength{\tabcolsep}{0.35pt}
  \renewcommand{\arraystretch}{1.05}
  \footnotesize
  \begin{tabularx}{\textwidth}{@{}
    >{\fontsize{7pt}{7.4pt}\selectfont\raggedright\arraybackslash}p{0.14\textwidth}
    *{3}{>{\centering\arraybackslash}X}
    >{\centering\arraybackslash}p{0.065\textwidth}
    *{3}{>{\centering\arraybackslash}X}
    >{\centering\arraybackslash}p{0.065\textwidth}
    *{3}{>{\centering\arraybackslash}X}
    >{\centering\arraybackslash}p{0.065\textwidth}@{}}
    \hline
    \multirow{2}{=}{\centering Experiment} &
    \multicolumn{3}{c}{\scriptsize BC Recall (H=33, 95\% CI)} &
    \multirow{2}{=}{\centering\scriptsize Preference ELO} &
    \multicolumn{3}{c}{\scriptsize BC Failure Recall (H=33, 95\% CI)} &
    \multirow{2}{=}{\centering\scriptsize Preference ELO} &
    \multicolumn{3}{c}{\scriptsize Push Away Recall (H=33, 95\% CI)} &
    \multirow{2}{=}{\centering\scriptsize Preference ELO} \\
    \cline{2-4}\cline{6-8}\cline{10-12}
    & \raisebox{-0.12em}[0pt][0pt]{EE} & \raisebox{-0.12em}[0pt][0pt]{TEE} & \raisebox{-0.12em}[0pt][0pt]{Joint} & &
    \raisebox{-0.12em}[0pt][0pt]{EE} & \raisebox{-0.12em}[0pt][0pt]{TEE} & \raisebox{-0.12em}[0pt][0pt]{Joint} & &
    \raisebox{-0.12em}[0pt][0pt]{EE} & \raisebox{-0.12em}[0pt][0pt]{TEE} & \raisebox{-0.12em}[0pt][0pt]{Joint} & \\
    \hline
    PPO-300K & \metric{.275}{.013} & \metric{.400}{.024} & \metric{.224}{.012} & \arenascore{843}{155} &
    \metric{.554}{.019} & \metric{.491}{.028} & \metric{.432}{.026} & \arenascore{1040}{185} &
    \metric{.332}{.023} & \metric{.282}{.026} & \metric{.198}{.018} & \arenascore{816}{128} \\
    PPO-3K & \metric{.267}{.011} & \metric{.322}{.022} & \metric{.203}{.011} & \arenascore{732}{155} &
    \metric{.517}{.020} & \metric{.424}{.030} & \metric{.378}{.028} & \arenascore{824}{171} &
    \metric{.297}{.023} & \metric{.190}{.017} & \metric{.147}{.015} & \arenascore{889}{132} \\
    PPO-3K+AIM500 & \bestmetric{.485}{.014} & \bestmetric{.403}{.023} & \bestmetric{.344}{.018} & \bestarenascore{1190}{167} &
    \bestmetric{.564}{.018} & \bestmetric{.454}{.031} & \bestmetric{.411}{.029} & \bestarenascore{1049}{185} &
    \bestmetric{.375}{.016} & \metric{.201}{.020} & \bestmetric{.165}{.016} & \bestarenascore{1029}{134} \\
    PPO-3K+Recovery500 & \metric{.272}{.012} & \metric{.378}{.026} & \metric{.215}{.012} & \arenascore{979}{155} &
    \metric{.545}{.022} & \bestmetric{.455}{.030} & \metric{.406}{.027} & \arenascore{989}{200} &
    \metric{.314}{.023} & \metric{.216}{.020} & \bestmetric{.169}{.016} & \arenascore{735}{128} \\
    PPO-3K+ICM500 & \metric{.222}{.011} & \metric{.393}{.025} & \metric{.186}{.011} & \arenascore{868}{145} &
    \metric{.506}{.022} & \metric{.441}{.029} & \metric{.383}{.027} & \arenascore{934}{186} &
    \metric{.267}{.020} & \bestmetric{.242}{.023} & \metric{.158}{.016} & \arenascore{924}{141} \\
    PPO-3K+DIAYN500 & \metric{.233}{.009} & \metric{.364}{.027} & \metric{.186}{.010} & \arenascore{693}{145} &
    \metric{.522}{.021} & \metric{.438}{.031} & \metric{.388}{.029} & \arenascore{976}{167} &
    \metric{.310}{.024} & \metric{.186}{.018} & \metric{.143}{.015} & \arenascore{741}{129} \\
    \hline
    Recovery-3K & \metric{.315}{.013} & \bestmetric{.422}{.026} & \metric{.266}{.014} & \arenascore{1059}{158} &
    \metric{.598}{.018} & \bestmetric{.512}{.029} & \bestmetric{.466}{.026} & \arenascore{1099}{202} &
    \metric{.396}{.018} & \bestmetric{.276}{.026} & \bestmetric{.226}{.018} & \arenascore{1250}{147} \\
    Recovery-3K+AIM500 & \bestmetric{.511}{.014} & \metric{.417}{.024} & \bestmetric{.361}{.019} & \bestarenascore{1412}{183} &
    \bestmetric{.613}{.016} & \metric{.498}{.028} & \metric{.459}{.026} & \bestarenascore{1147}{212} &
    \bestmetric{.421}{.019} & \metric{.264}{.024} & \metric{.223}{.018} & \bestarenascore{1317}{178} \\
    \hline
    AIM w/o re-estimation & \bestmetric{.526}{.013} & \bestmetric{.428}{.023} & \bestmetric{.379}{.019} & \bestarenascore{1388}{174} &
    \metric{.552}{.020} & \bestmetric{.464}{.031} & \bestmetric{.416}{.029} & \arenascore{1047}{169} &
    \metric{.357}{.017} & \metric{.226}{.022} & \metric{.183}{.018} & \bestarenascore{1092}{138} \\
    AIM w/o MI reward & \metric{.481}{.013} & \metric{.405}{.024} & \metric{.344}{.019} & \arenascore{1107}{208} &
    \bestmetric{.566}{.017} & \metric{.458}{.030} & \bestmetric{.416}{.028} & \arenascore{964}{151} &
    \bestmetric{.389}{.016} & \bestmetric{.231}{.023} & \bestmetric{.192}{.019} & \arenascore{1076}{207} \\
    AIM w/o Novelty reward & \metric{.483}{.014} & \metric{.401}{.025} & \metric{.340}{.019} & \arenascore{1376}{295} &
    \metric{.559}{.018} & \metric{.457}{.030} & \metric{.411}{.028} & \bestarenascore{1128}{191} &
    \metric{.377}{.016} & \metric{.224}{.023} & \metric{.182}{.019} & \arenascore{1048}{184} \\
    \hline
  \end{tabularx}
  \endgroup
  \end{minipage}}
  \caption{\textbf{Push-T collection comparisons and AIM ablations.}
  The row groups compare expert data scale and 500-episode additions to PPO-3K (top), reuse of the same AIM500 with Recovery-3K (middle), and AIM component ablations (bottom).
  BC-Failure interleaves BC and Push-Away behavior to introduce anomalous contacts.
  Joint recall is the per-transition minimum of end-effector and Tee recall, averaged over evaluation transitions. Higher recall and preference scores are better.
  }
  \label{tab:pusht-ablations}
\end{table*}

To test the robustness of our conclusions, we evaluate higher-resolution Push-T world models over a longer horizon (\(H=33\)), with end-effector, Tee, and joint \(\min\{\mathrm{EE},\mathrm{Tee}\}\) recall, and arena-style A/B human preference ratings~\citep{chiang2024chatbot} (1545 comparisons, 20 annotators).
We also add a BC-Failure test set that interleaves BC and Push Away policies as a proxy for suboptimal behavior during policy improvement.
Table~\ref{tab:pusht-ablations} summarizes the results.
Adding 500 AIM episodes improves BC joint recall by 69\% over PPO-3K and surpasses the 300K-expert baseline by 54\%,
and we find human preference is largely consistent with our recall
metrics (BC Arena Elo increases from 732 to 1190).

To test if AIM is exploiting some characteristic of the low variability of PPO-3K for performance gains, we ablate swapping the PPO-3K base data with Recovery-3K, while keeping the same AIM data collected with the PPO-3K reward model, improving BC joint recall by 36\% over Recovery-3K.
This suggests AIM targets action information not easily acquired through common passive approaches via play-like structured behavior.
We also ablate reward re-estimation by collecting 500 trajectories with the first learned reward model, improving BC joint recall by 10\% over iterative AIM.
This suggests that the initial reward model remains valid for longer than our analysis in Section~\ref{sec:implementation-details} predicted.
When we ablate the AIM reward terms independently, we see the novelty reward outperforming on the very off-policy Push Away task (16\% higher joint recall than AIM) where state coverage makes the most difference.
Meanwhile, MI-only reward scores a bit higher in human evaluation on BC and BC-Failure evaluation sets (+16\%/+8\% over AIM), collecting more discriminative action consequences near the expert distribution.
Together these results support the hypothesis that small data interventions can provide significant improvements in world model controllability and fidelity, while leaving ample room to further optimize this type of active collection to improve sample efficiency in robotics.

\section{Conclusion and Limitations}

\label{sec:conclusion}

In this work, we propose a new theoretical perspective relating the quality of an action conditioned world model to quality of its data: a world model is primarily limited by the action consequences explained by its data, not the scale of its data.
We then show we can directly optimize the coverage of action consequences using Action Influence Maximization (AIM), a theoretically motivated batch online reinforcement learning algorithm that directs a data collection policy to probe the environment in ways which are expected to have predictable consequences but are unexplained by the existing dataset.
AIM shows that data-focused active exploration in high-dimensional, pixel-based robotic environments is not only feasible, but results in as much as \(85\times\) data scaling efficiency compared to passive expert collection approaches.
The key insight is simple: scaling helps mainly if it scales action information, not simply the number of transitions.
This suggests a path toward data-efficient, intrinsic-motivation-driven continual improvement of generalist world models and embodied policies, where robots actively collect the data they need to improve their own capabilities.

\noindent \textbf{Limitations.} A key limitation of AIM is the implementation complexity and compute intensity of retraining multiple learned density reward models.
The no-reestimation ablation suggests fewer updates could suffice, and the learned reward could also be used offline to filter datasets~\citep{sener2018coreset}.
Simpler, more efficient implementation will be key in scaling AIM-like methods to pretraining scale.
Additionally, our sequential RL approximation is not guaranteed to achieve improvement in the AIM contextual bandits problem, suggesting further theoretical work may be needed to bridge this limitation in more complex settings where optimizing RL is more challenging.
Our analysis also leaves the effects of post-training on improving data efficiency unaddressed, which may allow for inference of action consequences absent from interaction data.
Our work represents a proof of concept of the power of active data collection for robotics in simulation, and transferring this to continual improvement of generalist models remains to be shown by future work.

\clearpage
\acknowledgments{
This research was supported in part by Microsoft Research and Toyota Research Institute (TRI); this article solely reflects the opinions and conclusions of its authors, and not TRI nor any other Toyota entity. We thank the Google TPU Research Cloud (TRC) Program, NVIDIA Academic Grant program, and Gemini Academic Program for their generous support of our research.
Samuel M. Bateman is supported by the National Science Foundation CISE Graduate Fellowships under
Grant No. 2313998. Any opinions, findings, and conclusions or recommendations expressed in this
material are those of the author(s) and do not necessarily reflect the views of the NSF.
The authors would like to thank Antonio Crivello, Duy Nguyen, Jeremy Bao, and Mingtong Zhang for providing invaluable feedback on the manuscript and presentation. 
Additionally, we would like to thank our amazing friends and colleagues for volunteering their time to help annotate world model rollout preferences.

}

\clearpage
\beginappendix{
\numberwithin{equation}{section}
\begingroup
\section{Proofs}
\label{app:story-proofs}

\subsection{Proof of Theorem~\ref{thm:story-data-bound}}
\label{app:story-data-bound-proof}
\par\endgroup

\begin{proof}[Proof]
\begingroup
Assume the KL prediction errors in
Equation~\eqref{eq:story-action-controllability} have finite expectations,
and the relevant log ratios are well defined with finite expectations.
Expanding the KL divergences in Equation~\eqref{eq:story-action-controllability}, we have
\begin{align}
    \hat I(\rho,\pi)
    &=\mathbb E_{\rho(s)\pi(a\mid s)p(s'\mid s,a)}
      \left[
        \log\frac{p(s'\mid s,a)}{\hat p_{\theta,\pi}(s'\mid s)}
        -\log\frac{p(s'\mid s,a)}{\hat p_\theta(s'\mid s,a)}
      \right]\notag\\
    &=\mathbb E_{\rho(s)\pi(a\mid s)p(s'\mid s,a)}
      [\log\hat p_\theta(s'\mid s,a)-\log\hat p_{\theta,\pi}(s'\mid s)].
\end{align}
\par\endgroup

We can then write the mutual information of the ground-truth transitions
over a given state marginal \(\rho\) as
\begin{equation}
    I(\rho,\pi)
    =\mathbb E_{\rho(s)\pi(a\mid s)p(s'\mid s,a)}
    [\log p(s'\mid s,a)-\log p_\pi(s'\mid s)].
\end{equation}
Using the variational posteriors,
\begin{equation}
\label{eq:story-app-variational-posterior}
\begin{aligned}
    q_\pi(a\mid s,s')
    &\triangleq\frac{p(s'\mid s,a)\pi(a\mid s)}{p_\pi(s'\mid s)},\\
    \hat q_\pi(a\mid s,s')
    &\triangleq\frac{\hat p_\theta(s'\mid s,a)\pi(a\mid s)}
                         {\hat p_{\theta,\pi}(s'\mid s)},
\end{aligned}
\end{equation}
we can write the difference between \(I(\rho,\pi)\) and
\(\hat I(\rho,\pi)\) as:
\begin{equation}
\begin{aligned}
    I(\rho,\pi)-\hat I(\rho,\pi)
    &=\mathbb E_{\rho(s)\pi(a\mid s)p(s'\mid s,a)}
    \left[\begin{aligned}
        &\log p(s'\mid s,a)-\log p_\pi(s'\mid s)\\
        &{}-\log\hat p_\theta(s'\mid s,a)+\log\hat p_{\theta,\pi}(s'\mid s)
    \end{aligned}\right].
\end{aligned}
\end{equation}
Applying Bayes rule to both expressions, and plugging in the identity
\(\rho(s)\pi(a\mid s)p(s'\mid s,a)
=\rho(s)p_\pi(s'\mid s)q_\pi(a\mid s,s')\), we have
\begin{align}
    I(\rho,\pi)-\hat I(\rho,\pi)
    &=\mathbb E_{\rho(s)p_\pi(s'\mid s)q_\pi(a\mid s,s')}
    \left[\begin{aligned}
        &\log q_\pi(a\mid s,s')-\log\pi(a\mid s)\\
        &{}-\log\hat q_\pi(a\mid s,s')+\log\pi(a\mid s)
    \end{aligned}\right]\notag\\
    &=\mathbb E_{\rho(s)p_\pi(s'\mid s)q_\pi(a\mid s,s')}
      [\log q_\pi(a\mid s,s')-\log\hat q_\pi(a\mid s,s')]\notag\\
    &=\mathbb E_{\rho(s)p_\pi(s'\mid s)}
      \left[\int_{\mathcal A}q_\pi(a\mid s,s')
      \log\frac{q_\pi(a\mid s,s')}{\hat q_\pi(a\mid s,s')}\,da\right]\notag\\
    &=\mathbb E_{\rho(s)p_\pi(s'\mid s)}
      \left[D_{\mathrm{KL}}\bigl(q_\pi(\cdot\mid s,s')
      \|\hat q_\pi(\cdot\mid s,s')\bigr)\right]\\
    &\geq 0,\notag
\end{align}
where the last inequality holds because of the non-negative property of
KL divergence~\citep{murphy2023probabilistic}.
This implies that the explained action information is upper bounded by the
true data MI \(\hat I(\rho,\pi)\leq I(\rho,\pi)\),
and dataset action influence limits world model explainable controllability
over the dataset.
\end{proof}

\begingroup
\subsection{Improving Replay Buffer Baseline Lower Bound Improves Action Information}
\label{app:story-aim-proof}
\begin{lemma}[{Improving Replay Buffer Baseline Lower Bound Improves Action Information}]
\label{lem:story-aim-improvement}
{Fix a state distribution \(\rho\) and replay policy \(\beta_k\).
Assume all expectations and KL divergences below are finite, and
proposed policies satisfy \(\pi(a\mid s)\ll\beta_k(a\mid s)\)
wherever \(\rho(s)>0\).
Equation~\eqref{eq:story-contextual-optimization} is lower bounded by the following objective:}
\begin{equation}
\label{eq:story-aim-surrogates}
    \max\nolimits_\pi\; L_k^\rho(\pi):=\mathbb E_{\rho(s)\pi(a\mid s)p(s'\mid s,a)}
    \left[\log\frac{q_k(a\mid s,s')}{\pi(a\mid s)}\right].
\end{equation}
{This lower bounds dataset mutual information for
all proposed \(\pi(a\mid s)\), and is tight at \(\beta_k\):}
\[
    L_k^\rho(\pi)\leq I(\rho,\pi),
    \qquad L_k^\rho(\beta_k)=I(\rho,\beta_k).
\]
{Thus improvement in the lower bound over \(\beta_k\)
improves dataset mutual information:}
\begin{equation}
\label{eq:story-aim-improvement}
    L_k^\rho(\pi)>L_k^\rho(\beta_k)
    \quad\Longrightarrow\quad
    I(\rho,\pi)>I(\rho,\beta_k).
\end{equation}
{The inequality holds under non-decreasing conditions as well.}
\end{lemma}
\par\endgroup

\begin{proof}[Proof]
Recall the definition of the variational posterior in
Eq.~\ref{eq:story-app-variational-posterior}. We can use it to rewrite
the ground-truth data mutual information as
\begin{equation}
    I(\rho,\pi)
    =\mathbb E_{\rho(s)\pi(a\mid s)p(s'\mid s,a)}
    [\log q_\pi(a\mid s,s')-\log\pi(a\mid s)].
\end{equation}
Using the variational lower bound proposed
by~\citet{barberagakovlowerbound2003}, we have
\begin{equation}
\begin{aligned}
    I(\rho,\pi)&\geq L_k^\rho(\pi)\\
    &=\mathbb E_{\rho(s)\pi(a\mid s)p(s'\mid s,a)}
    \left[\log\frac{\beta_k(a\mid s)p(s'\mid s,a)}{p_{\beta_k}(s'\mid s)}
    -\log\pi(a\mid s)\right],
\end{aligned}
\end{equation}
where \(q_k(a\mid s,s')
\triangleq\frac{\beta_k(a\mid s)p(s'\mid s,a)}{p_{\beta_k}(s'\mid s)}\)
is a variational posterior defined by the behavioral policy \(\beta_k\).

We note that the lower bound
\(L_k^\rho(\pi)\leq I(\rho,\pi)\) is tight if
\(\pi(a\mid s)=\beta_k(a\mid s)\) by plugging in the definition of
\(q_k(a\mid s,s')\): \(L_k^\rho(\beta_k)=I(\rho,\beta_k)\).
For any other policy \(\pi^+\) for which
\(L_k^\rho(\beta_k)<L_k^\rho(\pi^+)\), we have
\begin{equation}
    I(\rho,\beta_k)=L_k^\rho(\beta_k)
    <L_k^\rho(\pi^+)\leq I(\rho,\pi^+).
\end{equation}
For any valid policy improvement \(\pi^+\) on the
lower bound, the dataset mutual information will improve as well.
The same result holds with non-strict inequalities.

\end{proof}

\begingroup
\subsection{Proof of Theorem~\ref{thm:story-replay-optimizer}}
\label{app:story-replay-proof}
\par\endgroup

\begin{proof}[Proof]
\begingroup
Use the conditions in Appendix~\ref{app:story-aim-proof}, and assume
\(\rho\) and \(d_k\) are mutually absolutely continuous.
The replay-weighted contextual bandits problem is
\begin{equation}
\label{eq:story-replay-objective}
    \max_\pi\quad
    J_k^d(\pi):=\mathbb E_{d_k\pi p}[r_{\mathrm{AIM}}^k]
       -\mathbb E_{d_k}D_{\mathrm{KL}}(\pi\|\beta_k).
\end{equation}
\par\endgroup

We can define the log ratio as an intrinsic reward
\(r_{\mathrm{info}}^k(s,a,s')
=\log q_k(a\mid s,s')-\log\beta_k(a\mid s)\).
Rearranging terms in \(L_k^\rho(\pi)\), we have
\begin{equation}
    L_k^\rho(\pi)
    =\mathbb E_{\rho\pi p}
      [\log q_k(a\mid s,s')-\log\beta_k(a\mid s)]
      -\mathbb E_\rho D_{\mathrm{KL}}(\pi\|\beta_k),
\end{equation}
which is precisely the Lagrangian form of a KL-constrained optimization
problem with the Lagrange multiplier \(\eta=1\)~\citep{boyd2004convex}.
\begingroup
Its unique maximizer is a Boltzmann tilt of the following form~\citep{murphy2023probabilistic}:
\begin{equation}
    \pi_k^\star(a\mid s)\propto
    \beta_k(a\mid s)\exp\!\left(\mathbb E_p[r_{\mathrm{info}}^k]\right),
\label{eq:story-app-exponential-optimizer}
\end{equation}
where \(\mathbb E_p\) averages over \(s'\sim p(\cdot\mid s,a)\).
\par\endgroup

Separately, we can reweight the observed samples with a Boltzmann tilt of the state
marginal distribution:
\(\rho(s)/d_k(s)=\exp\!\left(\log\rho(s)-\log d_k(s)\right)\),
resulting in the following equality:
\begin{equation}
\begin{aligned}
    \rho(s)\beta_k(a\mid s)\exp\!\left(\mathbb E_p[r_{\mathrm{info}}^k]\right)
    &=d_k(s)\beta_k(a\mid s)
      \exp\!\left(\mathbb E_p[r_{\mathrm{info}}^k]+\log\rho(s)-\log d_k(s)\right)\\
    &=d_k(s)\beta_k(a\mid s)\exp\!\left(\mathbb E_p[r_{\mathrm{AIM}}^k]\right).
\end{aligned}
\end{equation}
These identical weights induce the same conditional policy
\(\pi_k^\star(a\mid s)\propto
\beta_k(a\mid s)\exp\!\left(\mathbb E_p[r_{\mathrm{AIM}}^k]\right)\),
which is also the statewise optimizer of \(J_k^d\).
Thus both problems share $\pi_k^*$ as an optimizer.

We apply Lemma~\ref{lem:story-aim-improvement} to the exponential optimizer, implying
\begin{equation}
    I(\rho,\pi_k^\star)-I(\rho,\beta_k)
    \geq L_k^\rho(\pi_k^\star)-L_k^\rho(\beta_k).
\end{equation}
Improvement is strict if \(\pi_k^\star\ne\beta_k\) on states with
positive evaluation probability.
\end{proof}

\section{Scaling Passive Data Collection}
\label{app:scaling-results}

\begin{figure}[htbp]
    \centering
    \includegraphics[trim=0bp 129.1621bp 394.4911bp 0bp,clip,width=0.95\textwidth]{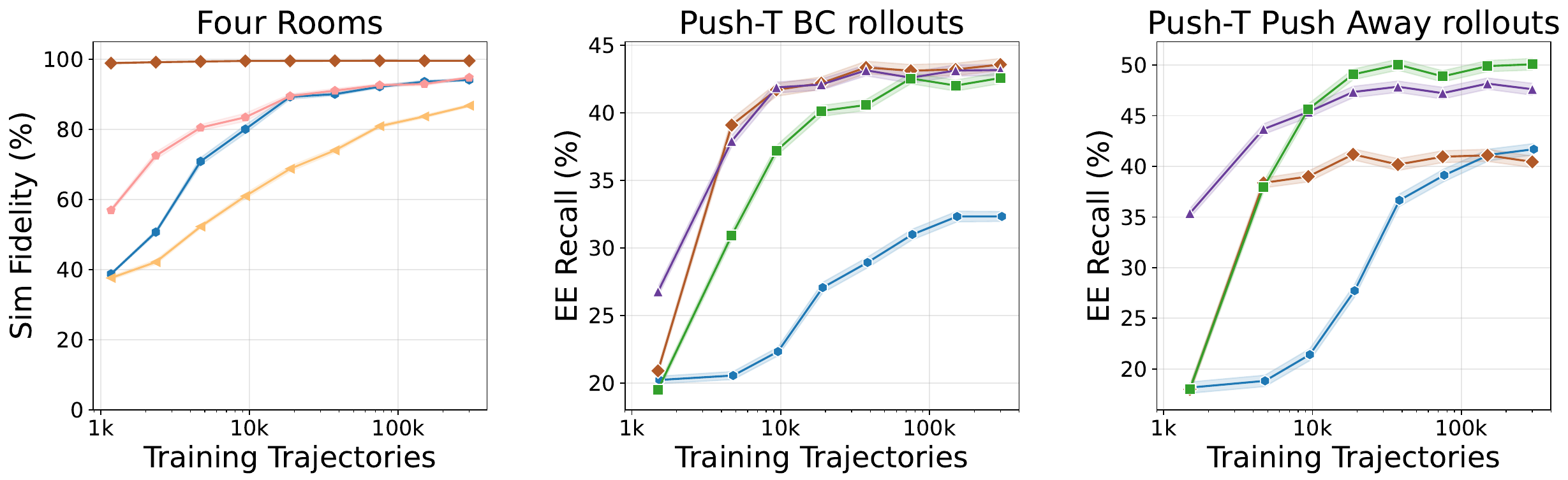}\\[2pt]
    \includegraphics[trim=0bp 0bp 0bp 362.3439bp,clip,width=0.95\textwidth]{figures/composite_figures/composite_scaling_controllability.pdf}

    \caption{
  \textbf{Data Scaling Trends for World Models:} We compare the simulation fidelity of world models trained on near-expert data (different types of PPO policies), recovery data, and random data on FourRooms~\cite{MinigridMiniworld23} and Push-T~\cite{taomaniskill3} environments.
  Four Rooms Sim Realism estimates predicted states through maximum likelihood and enumerates the state-action space exhaustively for full evaluation reference distribution coverage.
  For Push-T, such exhaustive coverage is infeasible, so instead we evaluate the pixel recall of the ground truth sim end effector mask on H=10 horizon rollouts with on-policy Behavior Cloning and off-policy "Push Away" expert rollouts.
    }
    \label{fig:scaling-policy-data}
\end{figure}

In Figure~\ref{fig:scaling-policy-data}, we analyze scaling behavior of passive data collection policies: in Four Rooms, uniform random data is both theoretically and practically strongest: uniform actions with uniform state resets provides near optimal evidence for maximizing explained action information score and thus controllability in fully observable, deterministic settings, approximating exhaustive coverage.
Expert-like PPO policies require much more data because their policy optimality purposefully underweights sampling many action informative alternate transitions.
Push-T shows the same qualitative lesson in a harder setting: broader non-expert data improves controllability, but the best data source depends on the specific eval rollout distribution.
These results show that a better data generating policy can be worth \(3\times\) the scaling efficiency per sample in improved controllability.
Random exploration works well when coverage is easy, but becomes less effective in high-dimensional manipulation with biased starts and partial observability.
Additionally, in Four Rooms, explained action information correlates strongly with state-extracted simulation fidelity when evaluated exhaustively over the state space, supporting our hypothesis that controllability is a useful surrogate for task-relevant fidelity.

\FloatBarrier

\section{Policy Details}
\label{apx:policy}
\label{app:story-report-figures}

\begin{figure}[htbp]
  \centering
  \begin{minipage}[t]{0.485\linewidth}
    \vspace{0pt}
    \begin{algorithm}[H]
      \color{black}
      \captionsetup{font=small,labelfont={color=black},textfont={color=black}}
      \caption{Batch AIM collection}
      \label{alg:aim-collection-implementation}
      \footnotesize
      \begin{algorithmic}[1]
        \Require Replay $\mathcal D_0$, density checkpoint $M_{-1}$, rounds $K$, batch sizes $B_k$
        \For{$k=0,\ldots,K-1$}
          \State Fit reward and annotate replay using Algorithm~\ref{alg:aim-reward-fitting}:
          \Statex \hspace{\algorithmicindent} $(M_k,\widehat{\mathcal D}_k)\gets
            \textsc{FitReward}(\mathcal D_k,M_{k-1})$
          \State Initialize collector $\pi_k$ from $M_k$'s behavior chunk policy; reset IQL critic and value heads
          \State Fit IQL, behavior cloning, and AWR flow matching on $\widehat{\mathcal D}_k$; enable AWR after critic warmup
          \State Collect $B_k$ environment trajectories $\mathcal T_k$ using $\pi_k$
          \State $\mathcal D_{k+1}\gets\mathcal D_k\cup\mathcal T_k$
        \EndFor
        \State \Return $\mathcal D_K$ for downstream world-model training
      \end{algorithmic}
    \end{algorithm}
  \end{minipage}\hfill
  \begin{minipage}[t]{0.485\linewidth}
    \vspace{0pt}
    \begin{algorithm}[H]
      \color{black}
      \captionsetup{font=small,labelfont={color=black},textfont={color=black}}
      \caption{Reward fitting and annotation}
      \label{alg:aim-reward-fitting}
      \footnotesize
      \begin{algorithmic}[1]
        \Require Replay $\mathcal D_k$, prior checkpoint $M_{k-1}$
        \Statex \textsc{FitReward}$(\mathcal D_k,M_{k-1})$:
        \State Warm-start from $M_{k-1}$; fit $\hat q_k$, $\hat\beta_k$, $\hat d_k$, and a behavior chunk policy on $\mathcal D_k$
        \State Jointly train divergence heads for the three reward densities
        \State Select checkpoint $M_k$ using held-out likelihood; freeze it
        \State Integrate flows and learned divergences to score all replay transitions
        \State Standardize each term using refreshed global replay statistics:
        \Statex $\hat r_k=z_k(\log\hat q_k-\log\hat\beta_k)-z_k(\log\hat d_k)$
        \State Annotate $\mathcal D_k$ with $\hat r_k$ to obtain $\widehat{\mathcal D}_k$, fixed for policy fitting
        \State \Return $(M_k,\widehat{\mathcal D}_k)$
      \end{algorithmic}
    \end{algorithm}
  \end{minipage}
\end{figure}

\begin{table}[tb]
\centering
\caption{Four Rooms Policy Statistics. All models are trained with PPO~\citep{schulman2017proximalpolicyoptimizationalgorithms} in Stoix~\citep{toledo2024stoix}.}
\label{tab:fourrooms-policy-statistics}
\begin{tabular}{l c c c}
\toprule
Expert Policy Name & Steps Trained & Entropy & Eval Return\\
\midrule
POMDP PPO & 300M & 0.054 & 0.657 \\
Early Stop PPO & 50M & 0.252 & 0.587 \\
State Oracle PPO & 300M & 0.022 & 0.839 \\
\bottomrule
\end{tabular}
\end{table}

\begin{table}[t]
  \centering
  \setlength{\belowcaptionskip}{10pt}
  \caption{Policy hyperparameters. All models share a frozen,
  per-domain–finetuned SD-VAE encoder~\citep{rombach2022stablediffusion} (bf16 activations, fp32 params),
  $144{\times}144$ inputs encoded to $18{\times}18{\times}4$ latents
  ($81$ tokens/camera), 2D sin--cos position embeddings
  ~\citep{vaswani2017attention, mildenhall2020nerf}, and pre-LN
  transformer blocks with QK-norm~\citep{henry-etal-2020-query} and SwiGLU FFNs~\citep{shazeer2020gluvariantsimprovetransformer}. IQL~\citep{kostrikov2022iql} variants add a
  cross-attention decoder on the shared BC backbone. Only the VAE is frozen across tasks.}
  \label{tab:policy-hparams}
  \small
  \resizebox{\linewidth}{!}{\begin{tabular}{l cccc}
    \toprule
    & Softmax BC & Diffusion BC & Softmax IQL+AWR & Diffusion IQL+AWR \\
    \midrule
    \multicolumn{5}{l}{\emph{Task}} \\
    Environment        & FourRooms~\citep{nikulin2024xlandminigrid, MinigridMiniworld23} & Push-T~\citep{taomaniskill3}        & FourRooms   & Push-T \\
    Action space       & 6 discrete  & 6-D cont.    & 6 discrete  & 6-D cont. \\
    \midrule
    \multicolumn{5}{l}{\emph{Architecture}} \\
    Model dim          & 256 & 256 & 256 & 256 \\
    Encoder layers     & 6   & 12  & 6   & 12 \\
    Attention heads    & 8   & 8   & 8   & 8 \\
    FFN ratio (SwiGLU) & 4   & 3   & 4   & 3 \\
    IQL decoder layers & --  & --  & 6   & 6 \\
    Action chunk $C$   & --  & 16  & --  & 16 \\
    \midrule
    \multicolumn{5}{l}{\emph{IQL + AWR}} \\
    Expectile $\tau$              & -- & -- & 0.7      & 0.7 \\
    Discount $\gamma$            & -- & -- & 0.99     & 0.99 \\
    AWR inverse temp.\ $\beta$   & -- & -- & 0.1      & 0.1 \\
    AWR weight clip              & -- & -- & $[0,100]$ & $[0,100]$ \\
    AWR warmup (steps)           & -- & -- & $20\mathrm{k}{+}5\mathrm{k}$ ramp & $20\mathrm{k}{+}5\mathrm{k}$ ramp \\
    Target EMA $\tau_\mathrm{EMA}$ & -- & -- & 0.005  & 0.005 \\
    \midrule
    \multicolumn{5}{l}{\emph{Optimization}} \\
    Optimizer            & AdamW~\citep{loshchilov2018adamw} & AdamW & AdamW & AdamW \\
    Peak LR              & $1\mathrm{e}{-4}$ & $2\mathrm{e}{-4}$ & $1\mathrm{e}{-4}$ & $2\mathrm{e}{-4}$ \\
    Weight decay         & 0.01 & 0.01 & 0.01 & 0.01 \\
    Grad clip (norm)     & 1.0  & 1.0  & 1.0  & 1.0 \\
    LR schedule          & \multicolumn{4}{c}{warmup~\citep{goyal2018accuratelargeminibatchsgd} $\to$ cosine to 0} \\
    Train steps          & 30k & 100k & 50k & 50k \\
    Global batch (seq.)  & 256 & 256 & 256 & 256 \\
    Params (non-VAE)     & 19.5M & 30.6M & 26.9M & 44.6M \\
    \bottomrule
  \end{tabular}
}
\end{table}

\begin{table}[t]
  \centering
  \setlength{\belowcaptionskip}{10pt}
  \caption{SD-VAE~\citep{rombach2022stablediffusion} decoder-only finetuning hyperparameters. One recipe produces the
  per-domain tokenizers shared by all policy and world models; only
  augmentation and frame budget differ across domains.}
  \label{tab:vae-hparams}
  \small
  \begin{tabular}{l l}
    \toprule
    Hyperparameter & Value \\
    \midrule
    Initialization        & \texttt{sd-vae-ft-mse}~\citep{rombach2022stablediffusion}  \\
    Finetuned scope       & decoder + post-quant conv (49.5M); encoder frozen \\
    Image resolution      & $144\times144\times3 \to [-1,1]$ \\
    Latent grid           & $18\times18\times4$ ($8\times$ downsample) \\
    Latent scaling        & $0.18215$ \\
    \midrule
    Loss                  & MSE + LPIPS (VGG)~\citep{zhang2018perceptual}\\
    Loss weights          & 1.0 / 1.0, each EMA-normalized ($\alpha=0.99$)~\citep{hafner2025dreamerv4} \\
    Latent encoding       & posterior sample \\
    \midrule
    Optimizer             & AdamW~\citep{loshchilov2018adamw} \\
    Peak LR               & $1\mathrm{e}{-5}$ \\
    LR schedule           & 5k warmup~\citep{goyal2018accuratelargeminibatchsgd} $\to$ cosine to 0 \\
    Weight decay          & 0.001 \\
    Grad clip (norm)      & 1.0 \\
    Train steps           & 50k \\
    Global batch          & 96 sequences \\
    Augmentation          & FourRooms: color inversion + translation; Push-T: none \\
    Seed                  & 42 \\
    \bottomrule
  \end{tabular}
\end{table}

We train a number of policies throughout this work across both discrete and continuous domains.

For the discrete domain Four Rooms, we implement our expert policies in the Stoix RL framework~\citep{toledo2024stoix} via a Multi-Layer Perceptron (MLP) which consumes local field of view simulator state features and any extra provided state information (e.g. absolute position, goal position, and door positions for the Oracle variant). 
These policies predict a softmax action distribution which is then sampled from to provide stochasticity. 
These are then nominally trained for 300 million environmental transitions via PPO~\citep{schulman2017proximalpolicyoptimizationalgorithms} to achieve reward saturation (See \Cref{tab:fourrooms-policy-statistics} for details).
For our behavior cloning policies, we instead use a pixel-based softmax policy with a Stable Diffusion Variational Auto-encoder (VAE)~\citep{rombach2022stablediffusion, Kingma2013AutoEncodingVB} and a ViT~\citep{dosovitskiy2021an} on top to provide feature tokens which are cross attended to by a learned policy decoder token. 
Our IQL~\citep{kostrikov2022iql} implementation similarly takes the same ViT encoded tokens and cross attends with a readout decoder transformer which takes in a learned value token and two learned Q tokens for pessimism.
The output decoder tokens are fed into individual MLPs to readout the estimated policy or value estimate for usage.
The VAE encoder remains fixed throughout training, everything else is trained end-to-end.

For the continuous domain Push-T, we construct the expert policy based on the Maniskill~\citep{taomaniskill3} simulator state-based continuous PPO~\citep{schulman2017proximalpolicyoptimizationalgorithms} implementation and likewise use their dense Push-T reward to learn a policy.
For the "Push Away" Push-T task we invert the goal footprint reward to encourage placement of the tee on the table in non-overlapping positions with the goal, providing a near symmetric inverse of the original Push-T task in terms of value. 
We also add an intrinsic novelty reward $r^{\text{nov}}_{t,e}$ to encourage diverse positions. It is implemented as the following
\begin{align}
    r^{\mathrm{nov}}_{t,e} &= \min_{j \in \mathcal{T}_{i,t}} \bigl\lVert \mathbf{x}^T_{i, t} - \mathbf{x}^T_{j}\bigr\rVert_2
\end{align}
which is computed by the minimum distance to a random subset of $n=32$ states sampled uniformly from all parallel trajectory rollouts, thereby acting like a pairwise contrastive term~\cite{eysenbach2022contrastive, handsell2006contrastivelearning} to encourage diverse trajectory states.
BC Failure data is generated by a geometrically distributed switching between Push-T BC and Push Away policies, with a 3:1 bias towards Push-T BC.
For the continuous behavior cloning policies, we use a diffusion policy using a similar Stable Diffusion VAE~\citep{rombach2022stablediffusion} and ViT~\citep{dosovitskiy2021an} backbone to the discrete implementation, which provides cross-attention context for a diffusion policy~\citep{Chi2023DiffusionPolicy} implemented via flow matching~\citep{lipman2023flow}. 
We parameterize actions as delta-actions with respect to the current robot end-effector frame.

The continuous IQL~\citep{kostrikov2022iql} implementation similarly uses Q and V decoder tokens in the same decoder structure.
Since we use an action chunking policy~\citep{zhao2023act}, its natural that the Q function must evaluate action chunks rather than just individual actions, and we adopt an action chunking n-step bootstrap target similar to Q-Chunking~\citep{li2026qchunking}.
The input actions are normalized and embedded using cosine positional encoding embeddings~\citep{mildenhall2020nerf, vaswani2017attention} to provide tokenization, and added to learned decoder query tokens, where each token represents one timestep in the chunk.
These tokens are then concatenated with the state context tokens to enable Q estimation. 
Policy extraction is then easy, where a single AWR~\citep{peters2010relativeentropypolicysearch, peng2019awr} weight is applied in flow matching to the whole target action chunk.
Further information on hyperparameters used in training and finetuning the VAE, Policies, and Critics can be found in \Cref{tab:vae-hparams} and \Cref{tab:policy-hparams}.

Finally, as mentioned in \Cref{sec:implementation-details}, we apply z-score normalization to the raw densities with the population statistics before using them as rewards due to their significant differences in magnitude resulting from the high observation dimensionality and the relatively low action dimensionality.
Z-score normalization can be seen as baselining of the rewards and power scaling of each of the Boltzmann tilt reward terms, and renders each of the rewards dimensionless which makes hyperparameter tuning of reward weights relatively easy and hyperparameter insensitive.
An additional perspective is that this normalization is a tunable choice for a more restrictive trust region of the policy update.
This results in the following practical learned AIM density reward used in our experiments:
\begin{align}
    \hat{r}_{\text{AIM}}(s,a,s') = \text{z\_score}(\log \hat{q}_{\beta_k}(a|s,s') -\log \hat{\beta_k} (a|s)) - \text{z\_score}(\log \hat{d_k}(s)).
\end{align}

\section{World Model Architectural Details}
\label[appendix]{apx:wm-arch}

\label{sec:model_architecture}

\begin{figure*}[htbp]
    \centering
\includegraphics[width=0.95\textwidth]{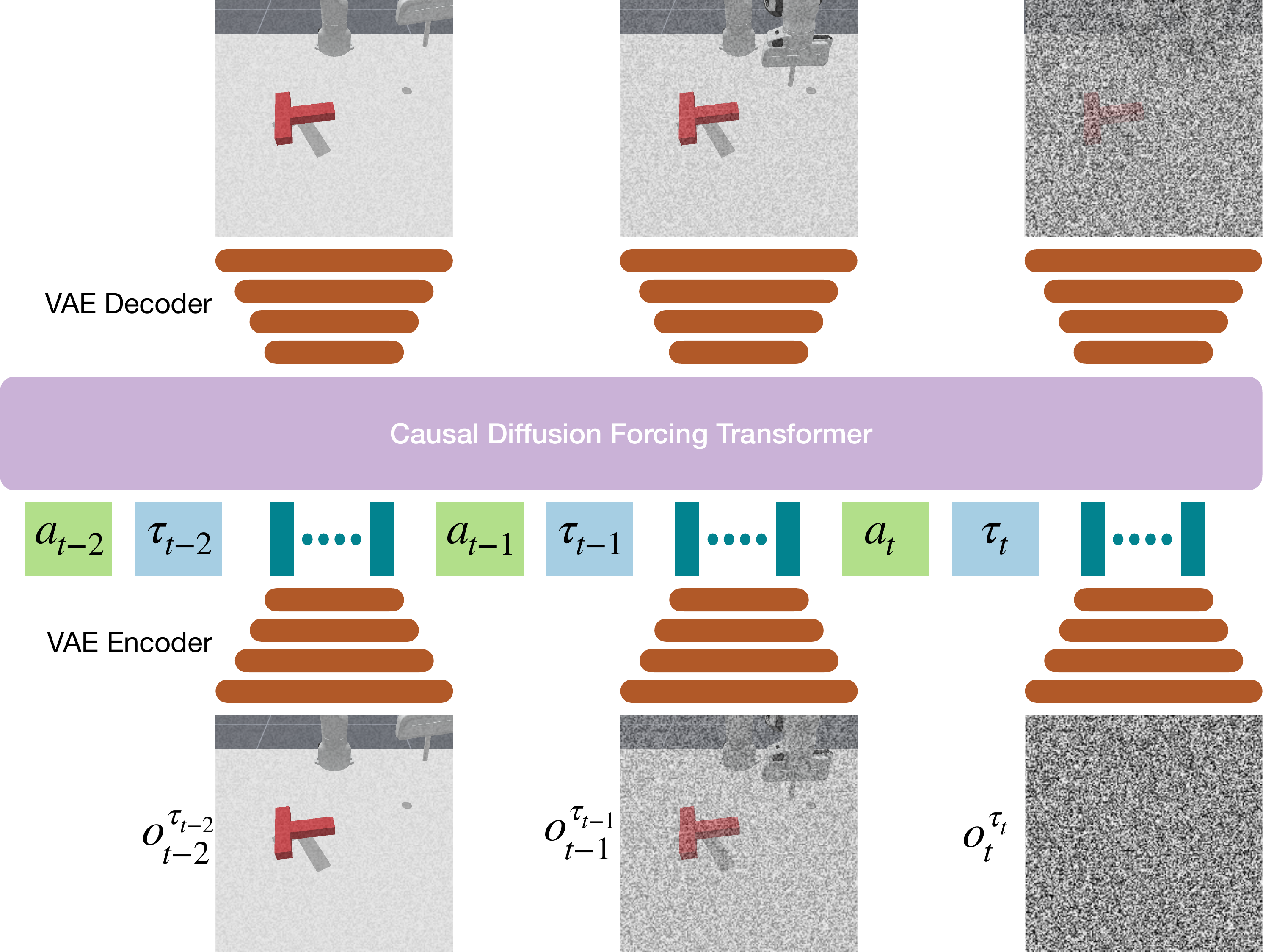}

    \caption{
    \textbf{Latent Diffusion Forcing Model Architecture:} 
    We train DreamerV4-like~\citep{hafner2025dreamerv4} latent diffusion forcing~\cite{chen2025diffusionforcing} action-conditioned causal autoregressive video world models with VideoRoPE~\citep{wei2025videorope} positional encodings and Stable Diffusion VAE~\cite{rombach2022stablediffusion} latent representations.
    These represent a strong world model architectural baseline which is held fixed in architecture and training recipe across generated datasets. These models generate human interpretable videos which help to identify degradations in the model, allowing us to focus on the impact of the data the model is trained on rather than specific architectural details.
    }
    \label{fig:diffusion-forcing-model-arch}
\end{figure*}

\begin{table}[t]
  \centering
  \setlength{\belowcaptionskip}{10pt}
  \caption{World model hyperparameters. Both share a frozen,
  per-domain–finetuned SD-VAE (bf16 activations, fp32 params),
  $144{\times}144$ inputs encoded to $18{\times}18{\times}4$ latents, pre-LN
  transformer blocks with QK-norm and SwiGLU FFNs, and a flow-matching
  velocity-MSE objective (EMA-normalized, $\alpha=0.99$). Neither uses a
  target network; only the VAE is frozen. For human evaluation and ablations
  in Section~\ref{sec:aim-ablations}, we instead train the Push-T models
  at $256{\times}256$ inputs to make them easier to evaluate.}
  \label{tab:wm-hparams}
  \small
  \begin{tabular}{l cc}
    \toprule
    & Latent Diffusion WM~\citep{peebles2023dit} & Diffusion Forcing WM~\citep{hafner2025dreamerv4, chen2025diffusionforcing} \\
    \midrule
    \multicolumn{3}{l}{\emph{Setting}} \\
    Domain      & FourRooms~\citep{nikulin2024xlandminigrid, MinigridMiniworld23} & Push-T~\citep{taomaniskill3} \\
    Prediction  & single-transition next-frame & sequence-level + AR rollout \\
    \midrule
    \multicolumn{3}{l}{\emph{Architecture}} \\
    Model dim          & 512 & 512 \\
    Layers             & 12 enc.\ + 12 dec. & 12 (9 spatial / 3 temporal)~\citep{hafner2025dreamerv4} \\
    Attention heads    & 8 & 8 \\
    FFN ratio (SwiGLU) & 3 & 3 \\
    Position encoding  & 2D sin--cos~\citep{vaswani2017attention, mildenhall2020nerf} & VideoRoPE~\citep{wei2025videorope} \\
    \midrule
    \multicolumn{3}{l}{\emph{Inference}} \\
    Sampler            & Euler & Euler \\
    Denoising steps    & 128 & 128 \\
    \midrule
    \multicolumn{3}{l}{\emph{Optimization}} \\
    Optimizer            & AdamW~\citep{loshchilov2018adamw} & AdamW \\
    Peak LR              & $2\mathrm{e}{-4}$ & $2\mathrm{e}{-4}$ \\
    Weight decay         & 0.01 & 0.01 \\
    Grad clip (norm)     & 5.0 & 5.0 \\
    LR schedule          & 10k warmup~\citep{goyal2018accuratelargeminibatchsgd} $\to$ cosine to 0 & 10k warmup $\to$ cosine to 0 \\
    Train steps          & 100k & 100k \\
    Global batch (seq.)  & 256 & 128 \\
    Max sequence length  & 48 & 64 \\
    Params (non-VAE)     & 97.8M & 44.5M \\
    \bottomrule
  \end{tabular}
\end{table}

We use two latent diffusion models in our experiments, one each for discrete and continuous environments.
We implement all models in JAX~\citep{jax2018github} and for both environments we train the world models from scratch without pretrained weights to avoid data contamination.

In Four Rooms, we train a simple DiT~\citep{peebles2023dit} under a flow matching objective~\citep{lipman2023flow} with state context generated from a Stable Diffusion VAE encoder~\citep{rombach2022stablediffusion} fed through a ViT~\citep{dosovitskiy2021an} to tokenize the latent representation. 
Since the environment is fully observable, we only condition on the current state observation image through cross attention.
Actions are provided by concatenating a simple per-action learned entity embeddings~\citep{guo2016entityembeddingscategoricalvariables} to the cross attention context.

For the continuous and partially observable Push-T environment, we must incorporate means of temporal stability through historical context.
For that, we leverage a state-of-the-art latent Diffusion Forcing~\citep{chen2025diffusionforcing} model trained similarly with a flow matching objective~\citep{lipman2023flow} and a similar architecture to DreamerV4~\citep{hafner2025dreamerv4}.
Image observations are encoded with a Stable Diffusion VAE~\citep{rombach2022stablediffusion}, patchified with linear projection, and visualized with Stable Diffusion VAE decoders finetuned with in domain data.
Both its action and diffusion timestep tokenization are completed using positional encodings~\citep{vaswani2017attention, mildenhall2020nerf} fed through an MLP.
The causal transformer itself enforces a sequence causal self-attention to allow for inference time autoregressive rollouts using a VideoRoPE~\citep{wei2025videorope} positional encoding.
The pre-norm transformer is outfitted with modern transformer improvements such as QKNorm~\cite{henry-etal-2020-query} and SwiGLU~\citep{shazeer2020gluvariantsimprovetransformer}.
See \Cref{fig:diffusion-forcing-model-arch} for details.
Full information on the hyperparameters used for finetuning the VAE and training the world models can be found in \Cref{tab:vae-hparams} and \Cref{tab:wm-hparams} respectively.

\section{Experimental Setup}

\label{appendix wm details}

The objective of our experiments is to evaluate data-generating policies that improve the controllability of world models, rather than to compare specific world models or architectures. 
As such, all world models on a given environment are identical and trained with exactly the same schedule except for the data they are trained on.

In Four Rooms, we generate expert ``source" data with a uniform random policy and a number of PPO~\citep{schulman2017proximalpolicyoptimizationalgorithms} policy variants: Local Field of View/Partial Observability (denoted POMDP PPO), an early stopped POMDP PPO (Early Stop PPO), and a sim-state-privileged variant (Oracle PPO).
To evaluate performance, we sample n=5 random layouts, and then exhaustively evaluate every state-action pair, and match the predicted image with the maximum-likelihood state through exhaustive comparison to provide an accurate measure of simulator controllability and fidelity.

For Push-T, such exhaustive evaluation is not feasible or informative, so instead we evaluate on rollouts from three policy evaluation transfer targets: a weak behavior cloned (BC) diffusion policy~\citep{Chi2023DiffusionPolicy} trained on the original replay buffer, a ``Push Away" PPO~\citep{schulman2017proximalpolicyoptimizationalgorithms} policy with a dense reward encouraging diverse trajectories pushing the tee away from the traditional goal, and a BC Failure hybrid between the two to represent moderately weak learners.
This is meant to evaluate two primary target policy evaluation downstream targets: improving a BC policy with an on-policy world model, transferring a policy to a new task with an off-policy world model, and improving a weak learner which regularly makes non-expert mistakes.
To isolate the controllability of the world model itself, we evaluate each frame by the end-effector masked per-pixel mean absolute error recall (pixel threshold 10\% of range so each color is unique) with a simulator-privileged mask to isolate only end-effector control sensitivity, penalizing both poor position and poor reconstruction as a proxy for rollout stability.

Given the cost of training multiple seeds across so many models, we only train one model per dataset/scale/resolution.
For both \(H=10\) and \(H=33\), we evaluate each model with three inference seeds for each of 100 policy rollouts per evaluation distribution, using the same evaluator. We then perform bootstrap resampling of trajectories and inference seeds (\(n=10000\)) to capture the variance across evaluations within a model.
Our human arena A/B tests show two model-predicted rollouts alongside the corresponding ground-truth rollout to compare two models under a Bradley--Terry preference ELO rating with the prompt ``Which prediction more accurately represents the behavior of the robot in the ground truth video?'' (1545 comparisons, 20 users)~\citep{chiang2024chatbot}.

Push-T is a continuous state space where simple Brownian motion provides limited coverage, so for fair comparison we implement both uniform random and a random goal-reaching policy, in addition to a PPO expert and a recovery baseline that alternates between a random-walk policy and the PPO expert with geometrically distributed switching times, providing a synthetic RaC-style analogue~\citep{hu2025rac}.
To compare against the various kinds of intrinsic reward exploration methods, we implement DIAYN~\citep{eysenbach2018diyan}, DADS~\citep{Sharma2020dads}, ICM~\citep{pathak2017curiosity}, and METRA~\citep{park2024metra}. These methods represent the families of prior unsupervised exploration methods in RL, including mutual information skill discovery (DIAYN, DADS), curiosity (ICM), and metric-aware skill learning. 

We use a common batch-online protocol to train all of these methods: each agent is first pretrained on the same prior dataset of $3005$ trajectories (around $1.1\times10^5$
transitions) generated by an expert PPO policy, and then learns to explore using its own mechanism. We perform $2\times10^5$ gradient updates before collecting $100$ new trajectories using the learned exploratory policies. We repeat this process for $5$ rounds ($500$ online trajectories total). All these baselines share a single
off-policy SAC~\citep{Haarnoja2018sac} backbone: a tanh-squashed Gaussian actor and twin critics with $[512,256]$ hidden units over the $31$-dimensional Push-T state, and differ only
in the definition of the intrinsic rewards. Specifically, DIAYN~\citep{eysenbach2018diyan} maximizes the discriminability of $50$ discrete skills under a state discriminator. DADS~\citep{Sharma2020dads} learns a skill-conditioned mixture-of-Gaussians dynamics model ($z\in[-1,1]^8$, four mixture components) and rewards the mutual information between skills and next-state transitions. ICM~\citep{pathak2017curiosity} rewards forward-model
prediction error in a learned feature space (curiosity scale $\eta=50$, forward/inverse weight $\beta=0.3$). METRA~\citep{park2024metra} learns four continuous unit-norm
skills under a temporal-distance-preserving encoder with a dual Lipschitz
constraint (reward scale $10$). For DADS~\cite{Sharma2020dads} and METRA~\citep{park2024metra}, we restrict the intrinsic reward to the manipulation-relevant simulator state dimensions (the end-effector and tee poses)
and for ICM~\citep{pathak2017curiosity} and DIAYN~\citep{eysenbach2018diyan} to the tee pose alone, so that skills and curiosity
target object motion rather than proprioception. The
hyperparameters for all four methods are summarized in
Table~\ref{tab:exploration-hparams}.

\begin{table}[t]
\centering
\caption{Hyperparameters for the exploration baselines in Push-T.}
\vspace{0.5em}
\label{tab:exploration-hparams}
\resizebox{\textwidth}{!}{\begin{tabular}{lccc}
\toprule
Method & Intrinsic reward & Key hyperparameters & learning rates \\
\midrule
DIAYN~\citep{eysenbach2018diyan} & skill discriminability & $50$ discrete skills; disc.\ over T-block pose          & $3 \times 10^{-4}$ \\
DADS~\citep{Sharma2020dads} & skill--dynamics MI     & $z\in[-1,1]^8$, $4$ mixtures, $K{=}100$ prior samples   & $3 \times 10^{-4}$ \\
ICM~\citep{pathak2017curiosity}& forward-model error    & $\eta = 50$, $\beta = 0.3$, feature dim $256$           & $3 \times 10^{-4}$ \\
METRA~\citep{park2024metra} & temporal-metric skills & $4$ continuous skills, reward scale $10$, dual reg. & $1 \times 10^{-4}$ \\
\bottomrule
\end{tabular}
}
\end{table}

\section{Expanded Four Rooms Results}
\label{apx:expanded-four-rooms}
In addition to the results of \Cref{sec:results}, we further investigate three questions about the Four Rooms environment: 
\textbf{Q1:} What happens to world model controllability in the low data limit? 
\textbf{Q2:} What does data coverage look like when comparing Random vs AIM data collection?
\textbf{Q3:} What is the effect of batch online data collection on other measures of sim fidelity?

\subsection{World Models Can Collapse to Their Data Generating Policy in the Low Data Limit}

\begin{figure*}[htbp]
  \centering
  \vskip 0.2in

  \begin{subfigure}[t]{0.48\textwidth}
    \centering
    \includegraphics[width=\linewidth]{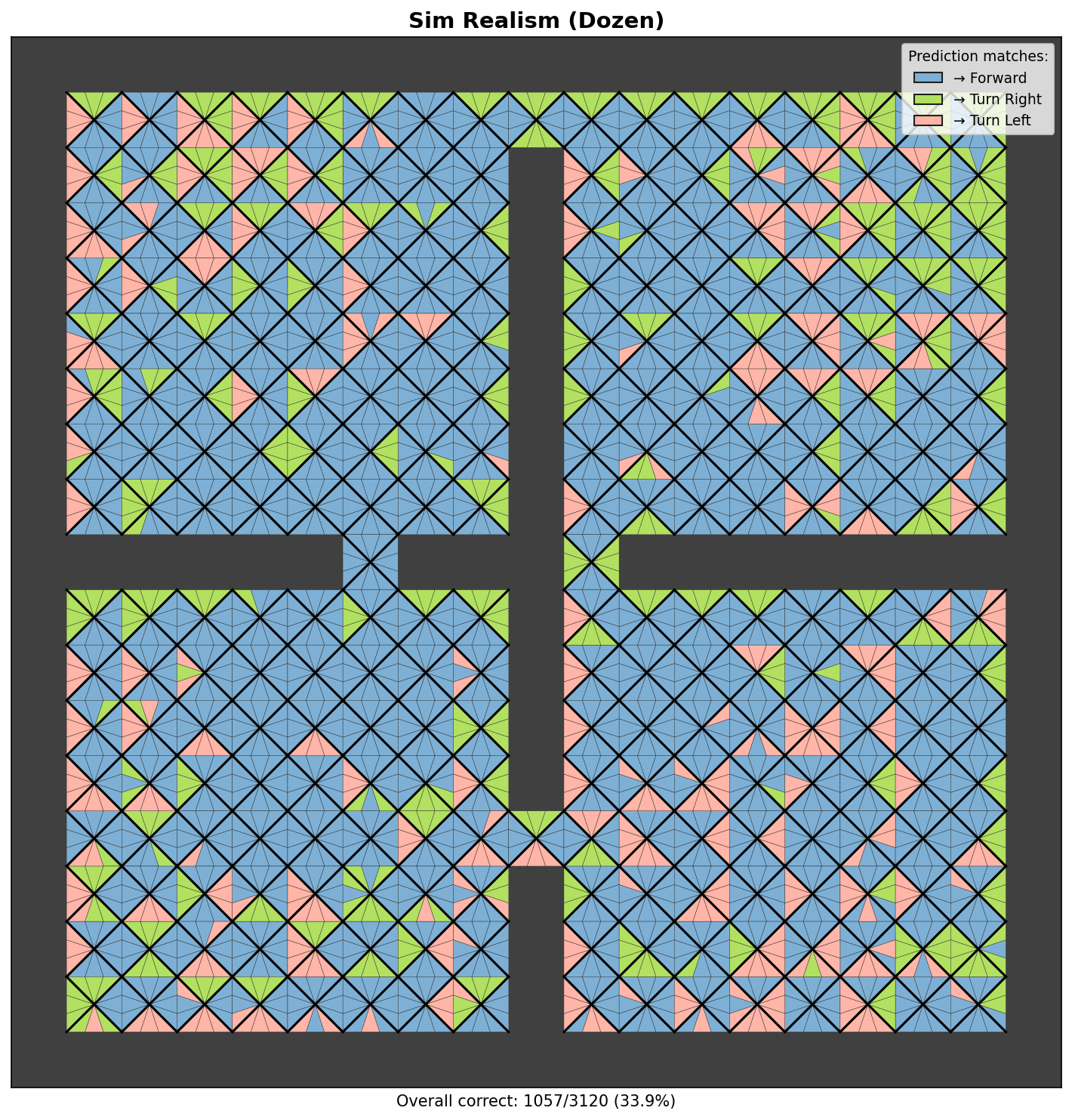}
    \caption{0.05\% dataset}
    \label{fig:early-stop-low-data}
  \end{subfigure}
  \hfill
  \begin{subfigure}[t]{0.48\textwidth}
    \centering
    \includegraphics[width=\linewidth]{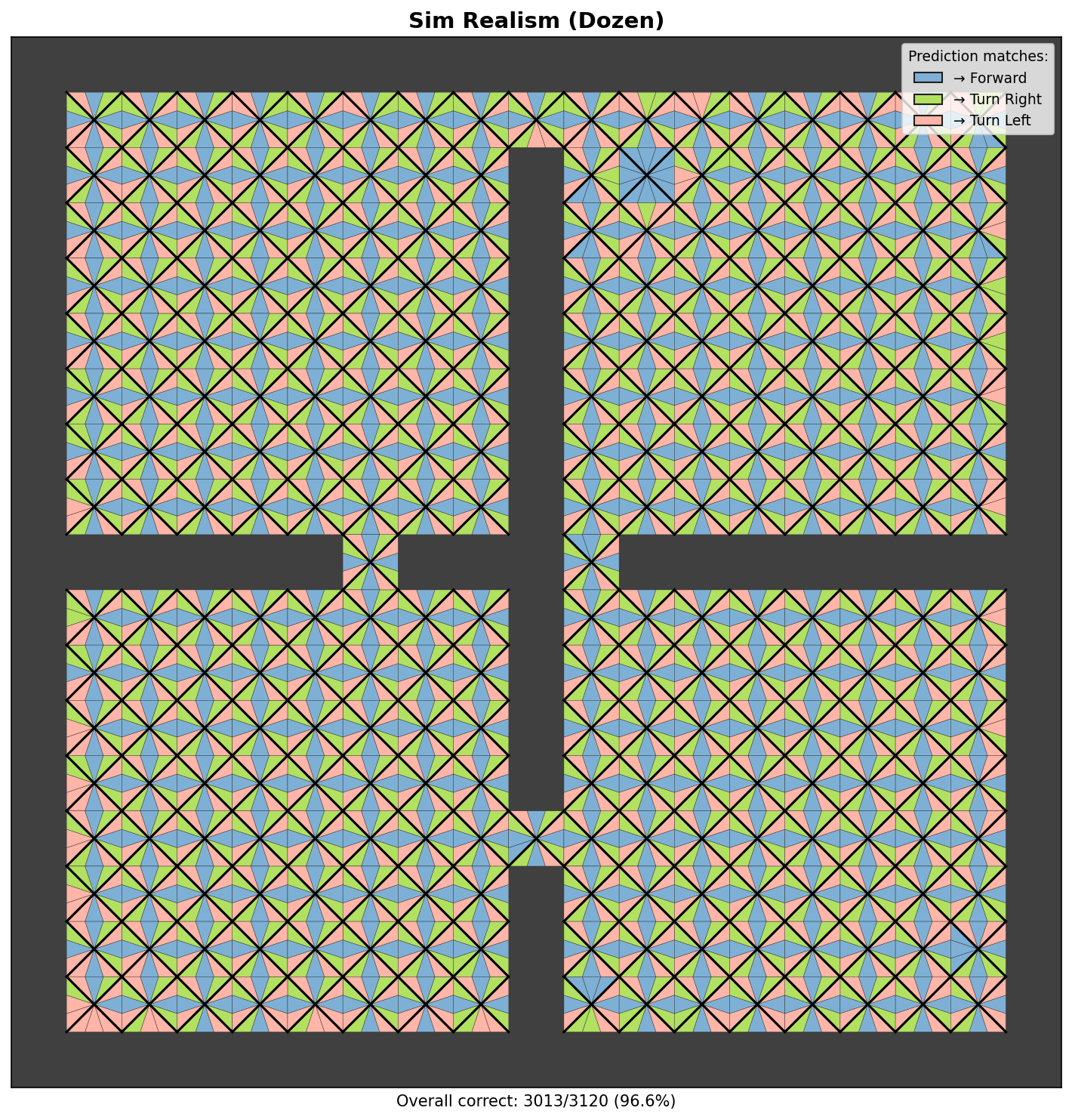}
    \caption{100\% dataset}
    \label{fig:early-stop-full-data}
  \end{subfigure}

  \caption{\textbf{Exhaustive Evaluation of World Model Action Sensitivity in Four Rooms:} This figure compares the evaluation results of one step action sensitivity of two data scales of a Four Rooms world model trained on expert data via the above ``Dozen plots". 
  Each quarter slice in the figure represents a state (grid position and direction of the agent), and within that slice it is split into 3 triangles representing each commanded Four Rooms action: forward, left, right.
  The color, as denoted in the top right, represents the action corresponding to the actual observed transition by the learned world model.
  One can see that, in low data regimes, the world model collapses to a near deterministic action policy where the model follows the most likely policy action regardless of action conditioning.
  However, the same model with more of the same data distribution becomes a near perfect action sensitive world model. 
  }
  \label{fig:fourrooms-pair-dozen-plots}
\end{figure*}

In \Cref{fig:fourrooms-pair-dozen-plots}, we explore the upper and lower limits of the quantity of expert data for training world models.
This figure shows an exhaustive evaluation of a Four Rooms world model trained on 150 trajectories compared to one on 300k trajectories, visualizing the expressive reproduction of action controllability in the 100\% data case, compared to the near non-existent action controllability of the low data regime on the left, where the world model transitions reproduce the underlying modes of the data generating policy near exactly under any action conditioning.
This proposes an intriguing hypothesis: the moderate action controllability of today's largest world models trained on seemingly endless internet video may actually be in the low data scaling regime of action-conditioned generative world models.
This further motivates our AIM intervention of generating more informative data to improve data scaling efficiency.

\subsection{AIM vs Random Fixed Start Data Coverage Ratios}
\label[appendix]{app:fourrooms-fixed-start-coverage}

\begin{figure}[htbp]
    \centering

    \begin{subfigure}[t]{0.20\linewidth}
        \centering
        \includegraphics[
            width=\linewidth,
            height=\linewidth,
            keepaspectratio
        ]{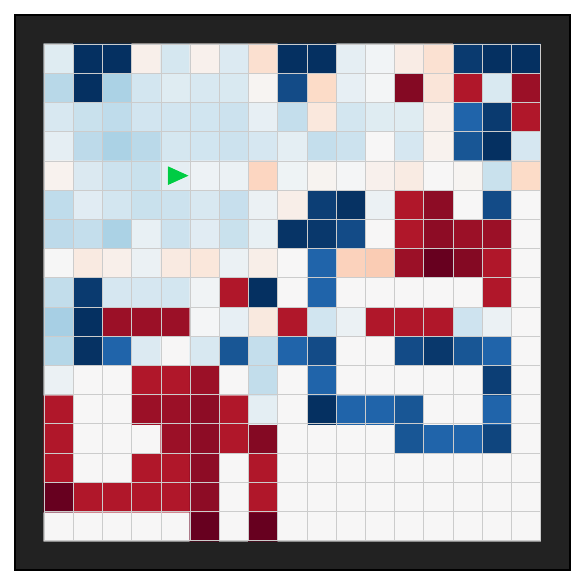}
        \caption{Iteration 1}
        \label{fig:appendix-3x1-a}
    \end{subfigure}
    \hfill
    \begin{subfigure}[t]{0.20\linewidth}
        \centering
        \includegraphics[
            width=\linewidth,
            height=\linewidth,
            keepaspectratio
        ]{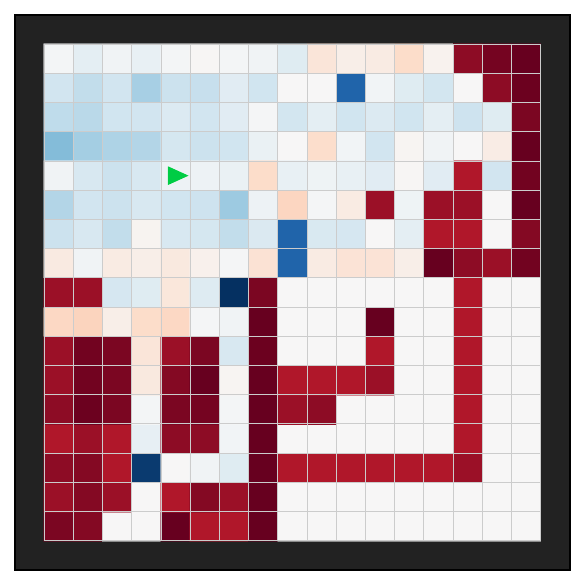}
        \caption{Iteration 2}
        \label{fig:appendix-3x1-b}
    \end{subfigure}
    \hfill
    \begin{subfigure}[t]{0.20\linewidth}
        \centering
        \includegraphics[
            width=\linewidth,
            height=\linewidth,
            keepaspectratio
        ]{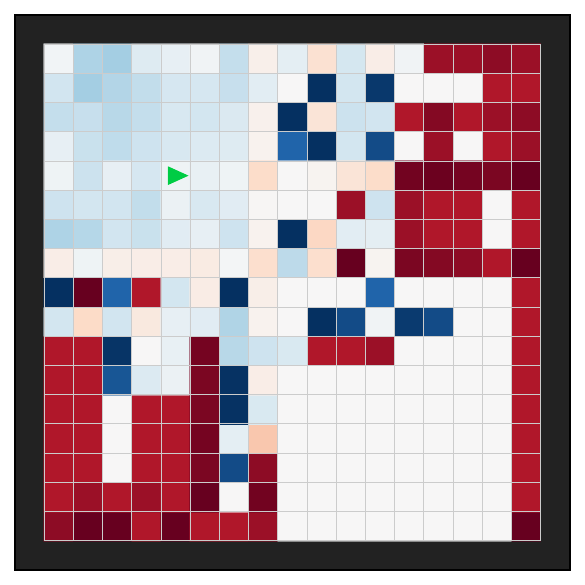}
        \caption{Iteration 3}
        \label{fig:appendix-3x1-c}
    \end{subfigure}
    \hfill
    \begin{subfigure}[t]{0.20\linewidth}
        \centering
        \includegraphics[
            width=\linewidth,
            height=\linewidth,
            keepaspectratio
        ]{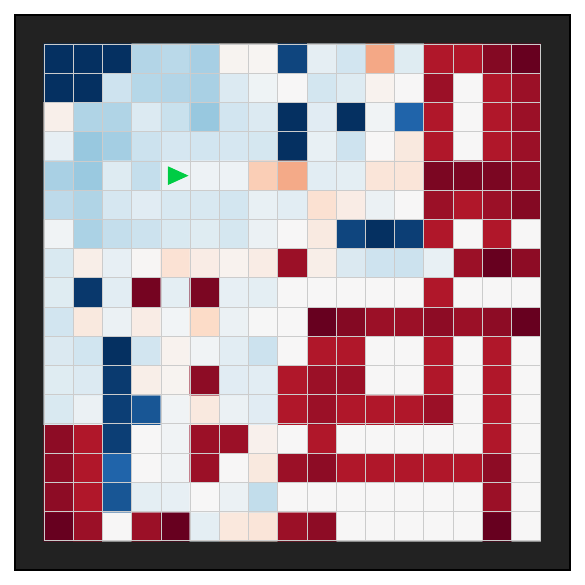}
        \caption{Iteration 4}
        \label{fig:appendix-3x1-d}
    \end{subfigure}
    \hspace{0.003\linewidth}
    \raisebox{0.00\linewidth}{
        \includegraphics[
            height=0.20\linewidth
        ]{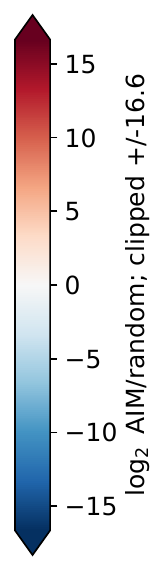}
    }
    \caption{
    \textbf{AIM-over-Random Fixed-Start Visitation Ratios:}
    In this figure, we fix the initial state of data collection rollouts to the pose of the green triangle in all figures. 
    Then, we compare state visitations in Four Rooms between rollouts generated by either successive AIM iterations or from a random policy as a heatmap of the clipped log ratio of their empirical state occupancy.
    More red represents more frequent AIM visitation, and more blue represents more frequent Random visitation.
    Pseudocounts with $\epsilon = 10^{-6}$ were added for normalization of the log ratios. 
    One can see AIM provides significantly better medium and long horizon state visitation compared to pure random, which leaves the initial top left room less frequently.
    Additionally, each iteration shows AIM providing coverage of largely different subsets of the state space, showing the value of AIM iteration and its ability to avoid revisiting previously explored regions.
    }
    \label{fig:appendix-3x1}
\end{figure}

One likely reason for the disagreement in optimal exploration between the low dimensional discrete Four Rooms and the higher dimensional Push-T task is that uniform random actions suffer from the curse of dimensionality as state-action dimensionality go up, making uniform state visitation support unlikely in a reasonable time horizon.
In contrast, AIM can explicitly perform goal-directed behavior to provide more uniform state visitation in aggregate.
To test this hypothesis, we compare state visitations in Four Rooms between rollouts generated by either successive AIM iterations or from a random policy, with all rollout episodes fixed to a predetermined initial state. 
This removes the influence of uniform state sampling which \Cref{fig:main-aim-results} shows is critical to the low rollout budget efficiency of uniform random actions in Four Rooms.
In \Cref{fig:appendix-3x1}, we visualize the ratio of state visitation between these two data collection policies.
The heatmaps show rollouts collected with successive AIM iteration policies are more likely to explore states farther from the start point, improving coverage and supporting the perspective that the policy actively explores the environment to search for novel state transitions that enhance the informativeness of the existing dataset.
Also of note is that AIM iterations cover largely non-overlapping sets of state visitation between iterations, showing that multiple batch iterations are useful to seek under-explored subsets of the state space and that the learned policy will largely avoid revisiting transitions it has already observed, precisely the desired behavior of a informative data acquisition policy.

\subsection{Effect of Batch Online Data Collection on Other Four Rooms Metrics}

\begin{figure*}[htbp]
    \centering
\includegraphics[width=0.95\textwidth]{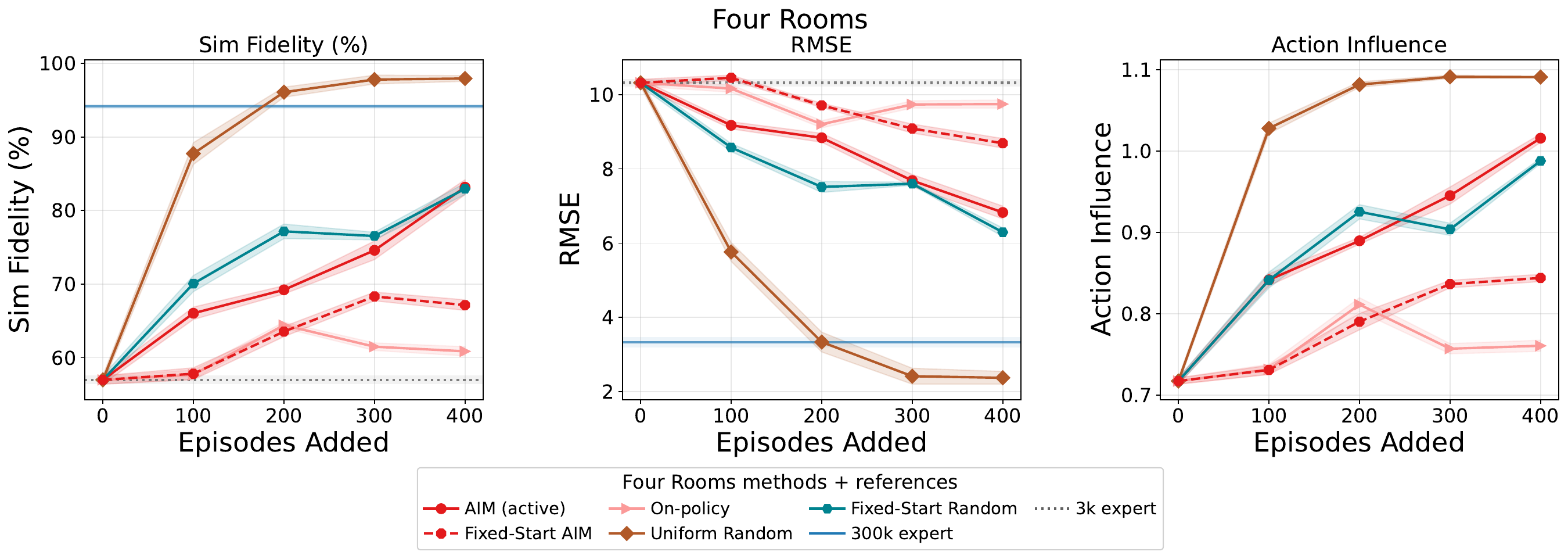}

    \caption{
    \textbf{AIM Four Rooms Extended Results:} 
    We provide the full side-by-side metric results from batch online active data collection in Four Rooms.
    Of note is the similar quality of RMSE as an ordered quality metric compared to the state-focused Sim Fidelity in this simple environment, which is in contrast to conventional wisdom around evaluating generative models.
    Additionally, we see that measured world model action influence grows proportionately with the evaluated sim realism, a metric computed by enumerating all actions, rolling them out in both the world model and the simulator, and evaluating their state extractive simulation consistency, weighted according to the action influence evaluation measure.
    }
    \label{fig:aim-fourrooms-extended-results}
\end{figure*}

In \Cref{fig:aim-fourrooms-extended-results}, we explore the effect of batch online data collection on other metrics.
One interesting detail is the strong correlation of pixel RMSE with sim fidelity, which is in contrast to broader expectations that generative models should be evaluated with measures of distributional similarity (e.g. FID~\citep{heusel2017fid}, FVD~\citep{unterthiner2019fvd}, etc.).
Instead, we find that, at least in this deterministic environment, RMSE pixel error is a strong proxy for state-extractive sim fidelity metrics. 
Our hypothesis is that this emerges in these simple deterministic environments because the only place where a model will make a generative mistake is where the interesting critical simulation state is most salient.
Thus, the pixel RMSE is largely focused on the same pixels as the state-extractive sim fidelity metric.

Additionally, we see that measured world model action influence grows proportionately with sim realism, reinforcing the relationship observed in \Cref{fig:fourrooms-action-information-fidelity} between world model controllability and our statistical action influence measure.
Another interesting observation is that observed action influence saturates at roughly the same rate as sim fidelity, indicating it may be a strong measure for simulator fidelity.
If explicitly optimized with a model's own generations (e.g. as a post-training RL objective), the action influence metric can easily be gamed by inventing new transitions or states which do not correspond with reality but satisfy the learned reward model.
Outside this pathological scenario and in conjunction with \Cref{fig:fourrooms-action-information-fidelity}, these indicate  that action influence may serve as a useful proxy for world model sim fidelity on real transition datasets when a high quality sim realism evaluation is not available.

\section{Expanded Push-T Results}
\label{apx:expanded-push-t}
Similarly to \Cref{apx:expanded-four-rooms}, we expand on our investigation from \Cref{sec:results} of the Push-T environment with two further questions:
\textbf{Q1:} Why do we need trust regions for density-based intrinsic motivation?
\textbf{Q2:} What is the effect of batch online data collection on other measures of sim fidelity?

\subsection{Trust Regions Suppress Out-of-Data Distribution Reward Estimation}

\begin{figure*}[htbp]
    \centering
\includegraphics[width=0.95\textwidth]{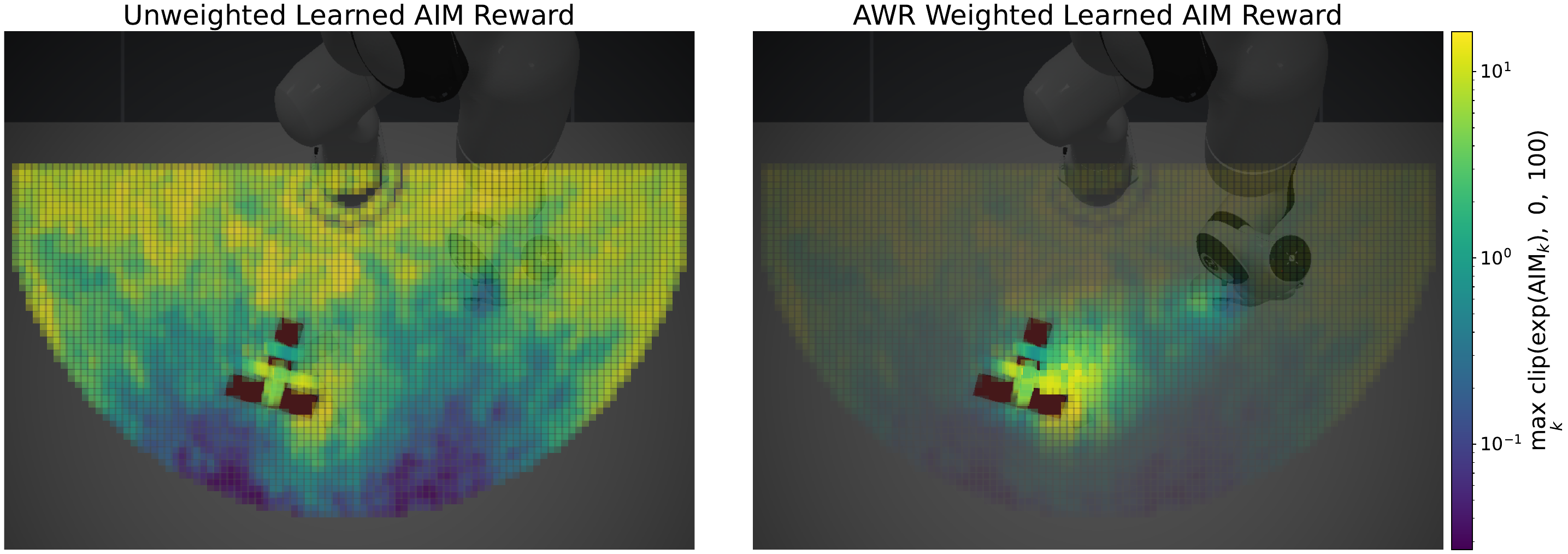}

    \caption{
    \textbf{The Need for Trust Regions for Density-Based Intrinsic Motivation in Push-T:} 
    In this experiment, we initialize the robot arm in the workspace in a grid and collect 4 transitions at each spot, one for each cardinal direction, then evaluate these transitions with the learned AIM reward model. 
    We then plot the max AIM reward on the above heatmap.
    We see that the learned AIM intrinsic motivation reward provides significant reward for regions outside of the nominal action workspace, and unbounded policy optimization would likely encourage unfocused state coverage similar to prior work~\citep{eysenbach2018diyan, Sharma2020dads, park2024metra}.
    In contrast, the AWR~\citep{peters2010relativeentropypolicysearch, peng2019awr} weighted reward focuses in a narrow region around pre-existing behavior, favoring task relevant information before broader exploration.
    }
    \label{fig:aim-reward-density-visualization}
\end{figure*}

In \Cref{app:story-aim-proof}, we show that AWR-style conservative KL constrained policy improvement~\citep{peters2010relativeentropypolicysearch, peng2019awr, levine2018controlasinference} is theoretically motivated for dataset action influence policy improvement.
Given prior mutual information-based skill diversity and exploration work~\citep{eysenbach2018diyan, Sharma2020dads, park2024metra} has used similar intrinsic motivation rewards and struggled to generalize to pixel observations and high dimensional action spaces, it seems worth investigating why this seemingly simple policy conservatism creates more focused exploration.
In \Cref{fig:aim-reward-density-visualization}, we evaluate the learned AIM reward on a grid of reachable states in the Push-T environment to visualize the underlying reward estimation from a fixed initial state condition.
We see that the learned AIM intrinsic motivation reward provides significant reward for regions outside of the nominal action workspace, and unbounded policy optimization would encourage unfocused state coverage, similar to prior work, toward regions that are less interesting for interaction but are otherwise highly rewarded for their diversity.
One possible explanation is that directly optimizing the AIM reward may result in the same diverse skills but relatively little environmental interaction pathology as prior work~\citep{eysenbach2018diyan, Sharma2020dads, park2024metra}, and that KL constrained policy optimization like AWR~\citep{peters2010relativeentropypolicysearch, peng2019awr} restricts the reward maximizing policy to regions near existing behavior priors, providing informative transitions which minimally deviate from prior expert data, focusing exploration to the most interesting, behavior-visiting regions first before exploring outward.
Additionally, another strength of evaluating our reward only in an offline RL~\citep{kostrikov2022iql, prudencio2023offlinerlsurvey} setting is that our learned density reward has reduced exposure to poorly calibrated out of distribution likelihood estimation which can occur in diffusion-style density modeling~\citep{kamkari2024geometricexplanationoflikelihoodoodparadox}, and that our distilled likelihood map~\citep{ai2026jointlikelihoodflowmaps} will never be queried out of distribution making its predictions lower error.

\section{Effect of Batch Online Data Collection on Other Push-T Metrics}

\begin{figure*}[htbp]
    \centering
\includegraphics[width=0.95\textwidth]{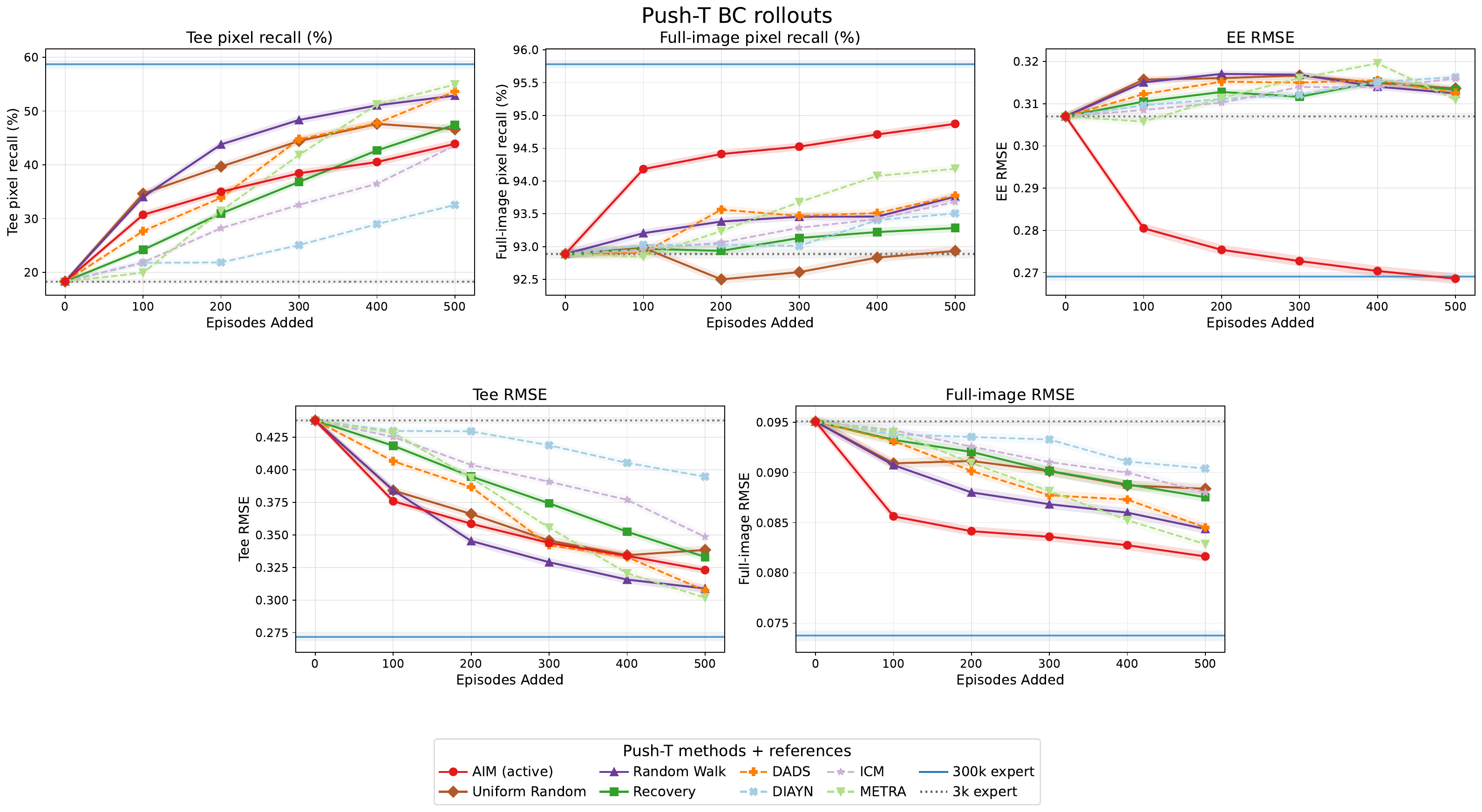}

    \caption{
    \textbf{AIM Behavior Cloning Push-T Extended Results:} 
    While End-Effector controllability is significantly improved by AIM, we see all model variants struggle with improving Tee pixel recall, a measure of simulator fidelity. 
    Nearly all exploration data collection policies we evaluate see noticeable improvement in tee recall, but even the top line 300k expert demonstration model struggles with performant tee manipulation.
    We hypothesize that this is a combination of poor coverage of tee positions in the dataset, as well as a lack of broad pre-training datasets in our experiments to improve extrapolation beyond the existing dataset.
    Additionally, we note that AIM shows significant improvement in full image pixel recall, outperforming the competing methods, and that pixel RMSE remains a strong metric of model fidelity in deterministic, nearly fully observable MDP settings despite the prevailing focus towards distributional evaluation metrics for generative world models~\citep{heusel2017fid, unterthiner2019fvd, ge2024fvdbias}.
    }
    \label{fig:aim-bc-pusht-extended-results}
\end{figure*}

\begin{figure*}[htbp]
    \centering
\includegraphics[width=0.95\textwidth]{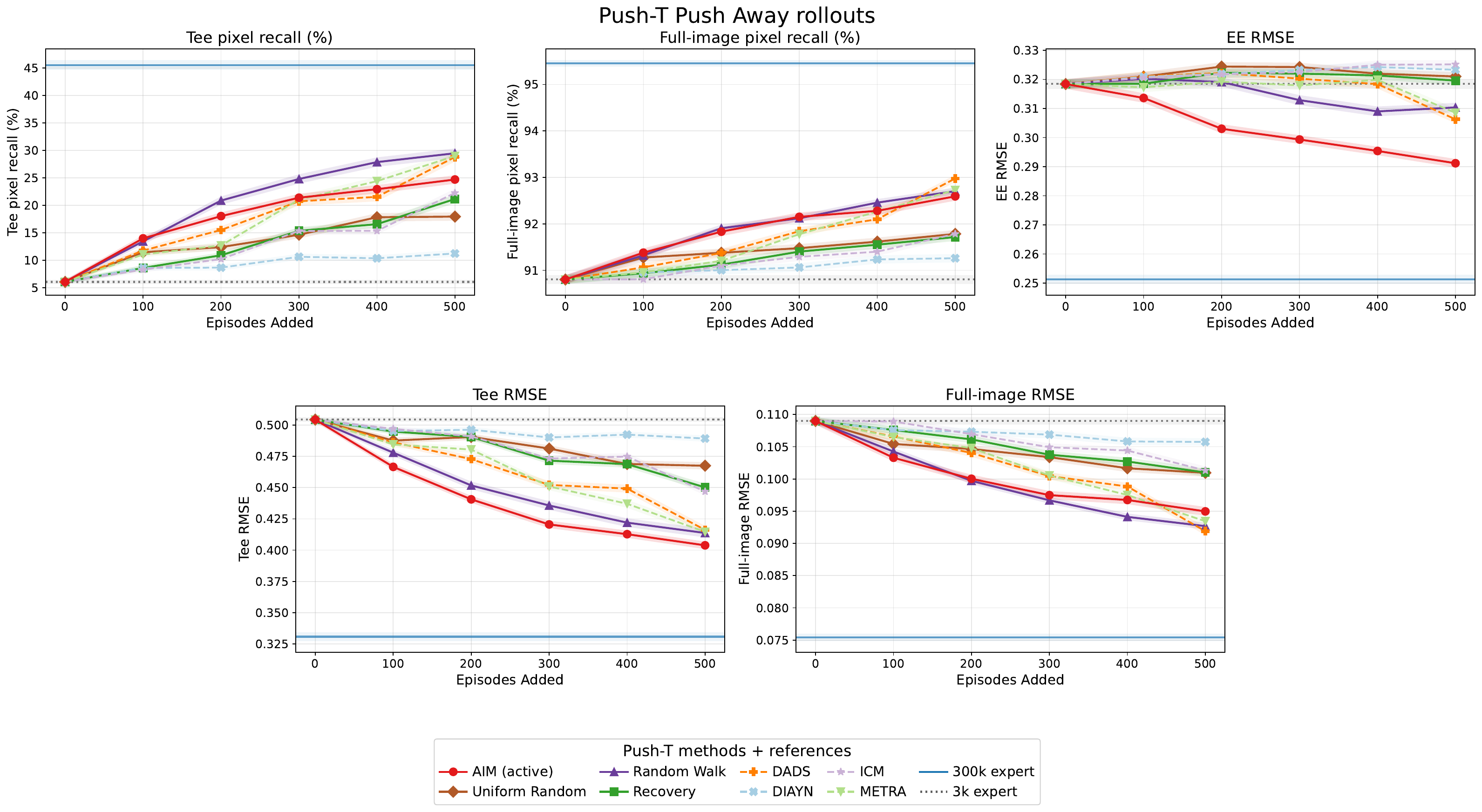}

    \caption{
    \textbf{AIM Anti-Push-T Extended Results:} 
    We see similar results to that of \Cref{fig:aim-bc-pusht-extended-results}, though dampened due to the significantly more off-policy nature of the ``Push Away" Push-T task which results in actions which are significantly more counter factual to those which exist in the pretraining data and thus harder to reproduce.
    This results in the worst data efficiency scaling from all of the active collection methods.
    }
    \label{fig:aim-anti-pusht-extended-results}
\end{figure*}

In \Cref{fig:aim-bc-pusht-extended-results} and \Cref{fig:aim-anti-pusht-extended-results}, we investigate the effects of batch online data collection on other measures of Push-T sim fidelity for both the Behavior Cloning and ``Push Away" policy evaluation distributions.
In particular, our controllability metric evaluates the controllability of the robot arm through end-effector pixel recall, which isn't necessarily the same as complete simulator fidelity in all cases.
We evaluate the low data regime scaling behavior of the tee pixel recall masked in the same way as the end-effector controllability measure.
It's clear that all data collection methods see similar performance improvements with additional off-policy data, with DIAYN~\citep{eysenbach2018diyan} being the lone underperformer.
Additionally, some of the intrinsic reward methods look particularly compelling on this metric, but this is likely due to their access to simulator state of the tee pose, as we found them unable to explore productively in the Push-T environment with only pixel observations.
However, qualitatively, even the highest data volume models struggle with tee fidelity on both evaluation distributions, and the differences are more stark in the off-policy ``push away" evaluation distribution in \Cref{fig:aim-anti-pusht-extended-results} where active data collection scaling is noticeably less efficient compared to the near expert behavior cloning data.
We hypothesize that better extrapolation to novel tee states would likely only happen with a rich prior data distribution to improve generalization, which we eschew in this work to prevent data contamination for precisely this reason. 
Overall, it seems pathological evaluation behaviors like ``push away" Push-T would remain difficult for world models to extrapolate to without noticeably more active data collection or some on-policy expert demonstrations to guide AIM exploration.

Additionally, similar to in \Cref{apx:expanded-four-rooms}, we find that full image pixel RMSE and recall is a surprisingly good realism metric in a deterministic world modeling setting, even when full observability is relaxed.
However, the disagreement between tee pixel recall ordering, end-effector controllability in \Cref{fig:main-aim-results}, and the full image pixel metrics begin to show full image pixel RMSE's lack of sensitivity to errors in semantically important regions of the image, motivating our targeted sim fidelity metrics and further work on effective evaluation of robotic world models.

}

\end{document}